\documentclass[letterpaper, times, mathptm, 10 pt, journal, twoside]{ieeeconf}
\IEEEoverridecommandlockouts
\usepackage{hyperref}
\usepackage{svg}

\usepackage{amsmath,amssymb,amsfonts}

\usepackage{algorithmic}
\usepackage{graphicx}
\usepackage{textcomp}
\usepackage{xcolor}
\usepackage{subcaption}
\usepackage[font=small,labelfont=bf]{caption} 

\usepackage{siunitx}

\usepackage[
    backend=biber,
    style=ieee,
    sorting=none,
    ]{biblatex}
\usepackage[normalem]{ulem}

\begin{document}
\newcommand{\vect}[1]{\boldsymbol{#1}}
\newcommand{\vecs}[1]{\boldsymbol{#1}}
\newcommand{\matr}[1]{\boldsymbol{#1}}
\newcommand{\matd}[1]{\mathcal{#1}}

\newcommand{\dotprod}[2]{\left\langle {#1}, \, {#2} \right\rangle}
\newcommand{\normdotprod}[2]{\frac{\left\langle #1, \, #2 \right\rangle}{\| #1 \| \, \| #2 \|}}

\newtheorem{theorem}{Theorem}[section]
\newtheorem{corollary}{Corollary}[section]
\newtheorem{lemma}{Lemma}[section]
\newtheorem{definition}{Definition}[section]

\newif\ifthesis
\thesisfalse

\newif\iflong
\longtrue

\iflong
 \newcommand{\editcolor}[1]{\textcolor{black}{#1}}
\else
\newcommand{\editcolor}[1]{\textcolor{black}{#1}}
\fi

\iflong
\else
\setlength{\headheight}{1mm}
\fi

\title{
\LARGE \bf
Passive Obstacle-Aware Control \\
to Follow Desired Velocities
}

\author{Lukas Huber$^{1}$, Thibaud Trinca$^{1}$, Jean-Jacques Slotine$^{2}$, Aude Billard$^{1}$
\thanks{
Manuscript received: January 21, 2024; 
Revised May 15, 2024;
Accepted June 27, 2024.
}
\thanks{This paper was recommended for publication by Editor Aniket Bera upon evaluation of the Associate Editor and Reviewers’ comments.}
\thanks{This work was supported by EU ERC grant SAHR.} 
\thanks{$^{1}$LASA Laboratory, Swiss Federal School of Technology in Lausanne - EPFL, Switzerland. \tt $\{$lukas.huber;aude.billard$\}$@epfl.ch }
\thanks{$^{2}$Nonlinear Systems Laboratory,  Massachusetts Institute of Technology, USA. \tt jjs@mit.edu}   
\thanks{Digital Object Identifier (DOI): see top of this page.}
}

\maketitle
\markboth{IEEE Robotics and Automation Letters. Preprint Version. Accepted June, 2024}{Huber \MakeLowercase{\textit{et al.}}: Passive Obstacle-Aware Control to Follow Desired Velocities.} 

\begin{abstract}
Evaluating and updating the obstacle avoidance velocity for an autonomous robot in real-time ensures robustness against noise and disturbances. A passive damping controller can obtain the desired motion with a torque-controlled robot, which remains compliant and ensures a safe response to external perturbations.
Here, we propose a novel approach for designing the passive control policy. 
Our algorithm complies with obstacle-free zones while transitioning to increased damping near obstacles to ensure collision avoidance. 
This approach ensures stability across diverse scenarios, effectively mitigating disturbances.
Validation on a 7DoF robot arm demonstrates superior collision rejection capabilities compared to the baseline, underlining its practicality for real-world applications. 
Our obstacle-aware damping controller represents a substantial advancement in secure robot control within complex and uncertain environments. \iflong \else \\
\editcolor{
The extended article with detailed proofs is available online }\parencite{passive2024huber_extended}. 
\fi
\end{abstract}

\section{Introduction}
Robots operating in unstructured, dynamic environments must balance adapting the path to the surroundings and following a desired motion. In human-robot collaboration, they must efficiently complete their tasks while ensuring safe compliance in the presence of potential collisions.

Velocity adjustments based on real-time sensory information are essential in complex and dynamic environments. Achieving this necessitates algorithms that are easily configurable and capable of rapid evaluation. Closed-form control laws eliminate the need for frequent replanning. Dynamical Systems (DS) is a valuable framework for representing such desired motion \cite{huber2019avoidance} where the behavior of the desired first-order DS is approximated through suitable controllers. 

Most robotic systems consist of rigid materials. Consequently, the interaction of these robotic systems with the surroundings leads to abrupt energy transfers, posing the risk of damage to the robot and its environment. However, modern robotic platforms are equipped with force and torque sensing capabilities that enable precise control over the forces exerted by the robot.
yet, integrating these sensors makes feedback controllers more complex. The robot must achieve its designated position by following a desired velocity profile while remaining compliant with interaction forces. Balancing position, velocity, and force constraints is addressed by \textit{impedance controllers} \iflong
\parencite{takegaki1981new, hogan1984impedance} \else \parencite{hogan1985impedance}\fi.

Obstacle avoidance is fundamental to motion control, with reactive approaches capable of handling dynamic and intricate environments \parencite{huber2019avoidance}. The controller should remain compliant in free space while adhering to the desired motion when getting close to the obstacle. Furthermore, when encountering surfaces like fragile glass on a table, the controller must adopt stiffness to prevent a collision, yet it should be compliant when interacting with an operator (Fig.~\ref{fig:table_avoidance_with_obstacle}).

\begin{figure}
\iflong
\centerline{\includegraphics[width=.95\columnwidth]{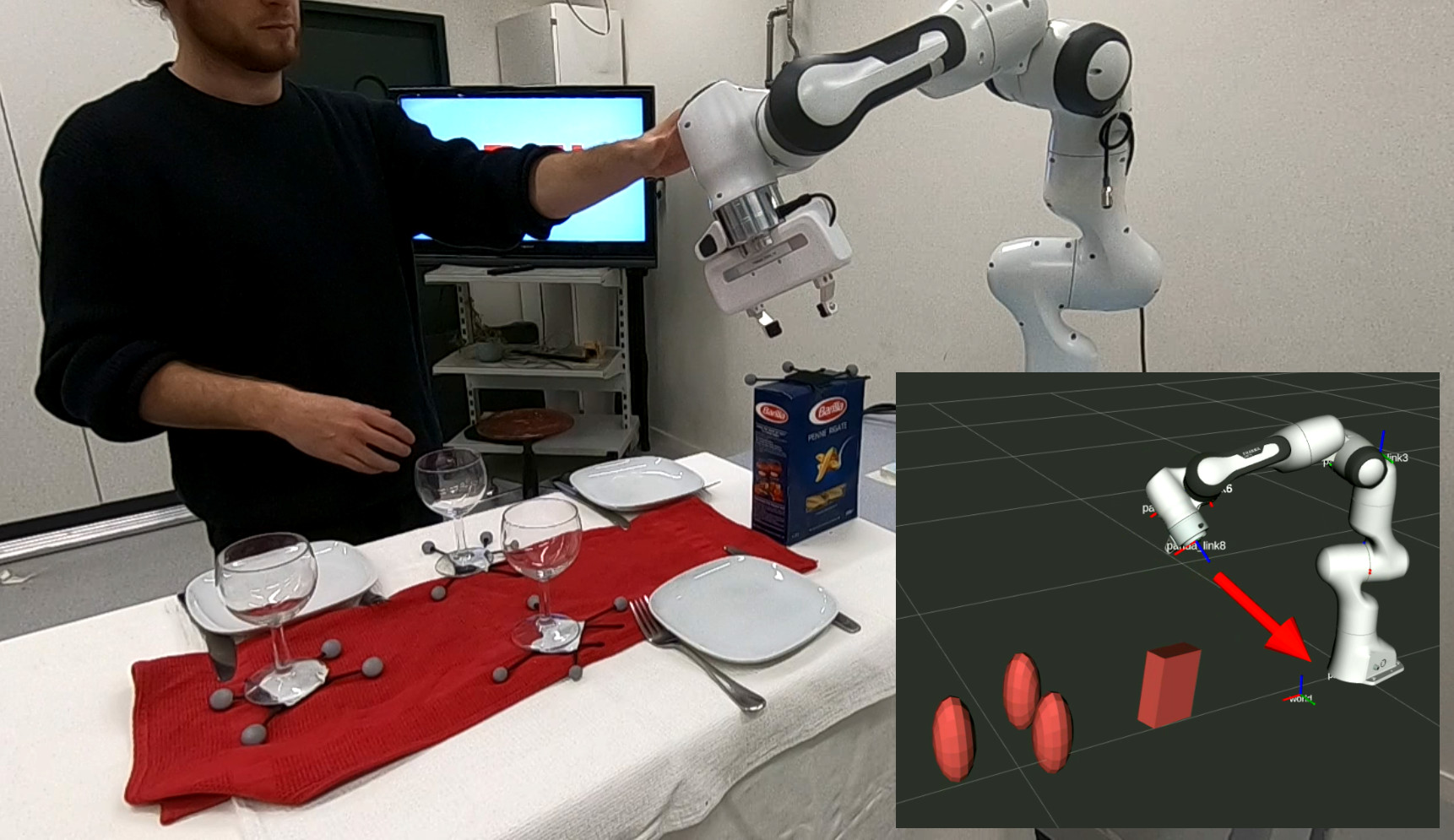}}
\else
\centerline{\includegraphics[width=.7\columnwidth]{figures/robot_arm_table_avoidance}}
\fi
\caption{
The proposed passive obstacle-aware controller lets the robot absorb external disturbances while ensuring collision avoidance. 
While tipping over the closed pasta box on this dinner table setup might be acceptable. Yet, the delicate wine glasses demand careful handling to prevent breakage.
}
\label{fig:table_avoidance_with_obstacle}
\end{figure}

This work introduces a novel approach incorporating dynamic obstacle avoidance using DS and variable impedance control, enhancing adaptability, reactivity, and safety in robotic movements. It empowers robots to navigate complex environments, proactively avoiding collisions and rejecting disturbances. Our approach is evaluated through with a 7-degree-of-freedom (7DoF) robot arm, demonstrating robust and safe control in real-world scenarios.

\subsection{Related Work}

\subsubsection{Force Control}
Impedance control, a powerful feedback algorithm, effectively applies Cartesian impedance to nonlinear manipulators' end-points \iflong \parencite{takegaki1981new, hogan1985impedance} \else \parencite{hogan1985impedance} \fi. The controller replaces the computationally intensive \textit{inverse kinematics} with the more straightforward \textit{forward kinematics}. Impedance control establishes a dynamic relationship between desired position, velocity, and force, offering a holistic control framework.
Initially, impedance controllers employed constant stiffness, but researchers have explored various dynamic control parameter approaches to enhance adaptability in complex environments \parencite{vanderborght2013variable, abu2020variable}. However, introducing dynamic parameters into the control framework requires taking special care of the system's stability.
Furthermore, admittance control is designed to adapt to external forces while remaining stable contact \parencite{glosser1994implementation}. Admittance control can be interpreted as a special type of impedance control.

Passive velocity controllers use a state-dependent velocity, converted into a control force through a damping control law. Since the controllers have variable damping parameters, stability can be guaranteed using storage tanks \parencite{li1999passive}. 
\iflong
Furthermore, complex control parameters can be learned from human demonstrations. By continuously adapting the parameters, the controller can be observed to improve its tracking performance in direct human-robot collaboration tasks \parencite{gribovskaya2011motion}.
\fi
However, high compliance often results in decreased motion following. Yet, carefully designing the damping matrix, which increases stiffness in the direction of motion but remains compliant otherwise, results in improved convergence \parencite{kronander2015passive}. 
Combining impedance controllers with admittance controllers can be used to increase accuracy in cooperative control
\parencite{fujiki2022series}.
However, these controllers' adaptations focus on improving movement accuracy rather than actively rejecting disturbances.

\iflong
Impedance controllers play an important role in interaction tasks, as in these situations, the robot is subjected to forces that are hard to predict precisely. However, they must be actively managed to ensure stable contact without damage occurring. In these situations, passive controllers using storage tanks based on the system's energy allow to limit contact force \parencite{kishi2003passive}.
A general framework can be used to ensure that position-, torque-, and impedance controllers exhibit passivity during interaction tasks. This is achieved by interpreting the torque feedback as the shaping of the motor inertia. It allows flexible robot arms for complex interaction tasks, such as insertion or wiping \parencite{albu2007unified}. 
By sensing the interaction force and actively adapting the trajectory based on the physical interaction force and the virtual repulsive force from the obstacle. This can be used for obstacle-aware motion generation \parencite{haddadin2010real}.
\fi

\iflong
Teleoperated systems of robot arms come with control delays and require a closed-loop controller of the robot that can adapt autonomously, ensuring stable and reliable performance. Passive controllers present themselves perfectly for such a task  \parencite{stramigioli2005sampled}. This method addresses the intricate dynamics of interactive systems, ensuring stable and reliable performance.
A similar approach involves slowly updating the desired position while incorporating a spring-damper model through an impedance controller, enabling seamless interaction with the teleoperated robotic system \parencite{lee2010passive}.
However, these models require human input for active collision avoidance.
\fi

Many impedance controllers with time-varying control rely on energy tanks to ensure stability. This is a virtual state, which is filled or emptied depending on the controller command. Limiting the energy tank to a maximum value ensures the system's stability  \parencite{ferraguti2013tank}. However, when this limit is reached, the controller is constrained and interferes with the controller's optimal functioning. This can result in the controller not achieving some functionalities, such as collision avoidance.
Alternatively, the impedance controller can be constrained by adapting the damping and stiffness and the rate of change based on a Lyapunov function \cite{kronander2016stability}.

\subsubsection{Obstacle Avoidance}
In dynamic environments, obstacle avoidance needs to be evaluated to fast ensure safe navigation. Repulsive force fields, that pointing away from obstacles create a collision-free motion \parencite{khatib1987unified}. 
However, these artificial potential fields are susceptible to local minima, which led to the introduction of navigation functions, i.e., a global function that creates repulsive fields while ensuring a single, global minimum \parencite{koditschek1990robot}. Such functions depend on the distribution of the obstacles and are hard to adapt to dynamic environments and high dimensional spaces \parencite{loizou2022mobile}.

Passive controllers can also track the desired motion of the artificial potential field while compensating for Coriolis and centrifugal forces \parencite{duindam2004passive}. 
Nonetheless, these methods lack the guarantee of disturbance repulsion around obstacles, which is addressed by the integration of circular fields, which rotates the robot's path around the obstacles \parencite{singh1996real}. 
This allows force-controlled navigation in cluttered environments, yielding convergence for simple obstacles \parencite{haddadin2011dynamic}. 
Conversely, repulsive fields can be combined with elastic, global planning \parencite{brock2002elastic} for improved convergence. This allowed adding a repulsive force from the obstacle, ensuring collision avoidance \parencite{tulbure2020closing}. 
However, methods based on artificial potential fields are prone to local minima in cluttered environments.

\iflong
Traditional tracking controllers often require complex linearization or simplification methods. However, a class of geometric tracking controllers enables exact control of nonlinear mechanical systems with low computational cost \parencite{udwadia2003new}. By reformulating control problems as a specific class of optimal controllers, this approach facilitates the derivation of standard control problems in robotics \parencite{peters2008unifying}. As a result, many geometric controllers directly output a control force and hence do not rely on an additional (impedance) controller.
Integrating local Riemannian Motion Policies (RMP) has led to globally stable force-controlled motion \parencite{cheng2020rmp}. Moreover, advancements in position-dependent Riemannian metrics allow for improved task design using RMP and reactive force control under constraints \parencite{bylard2021composable}.
Geometric fabrics have emerged as a valuable mathematical tool for shaping a robot's nominal behavior while capturing essential constraints like obstacle avoidance, joint limits, and redundancy resolution \parencite{xie2020geometric}. Combining Finsler geometry and geometric fabrics has further enhanced path consistency \parencite{ratliff2021generalized}.
Integrating geometric fabric methods into classical mechanical systems has enabled various physical behaviors, notably exemplified in multi-obstacle avoidance for a 7DoF robot arm \parencite{van2022geometric}. Despite these significant achievements, the parametrization of geometric methods and their application to general systems remains challenging and hinders their wide acceptance.
\fi

Harmonic potential functions can ensure the absence of local minima in free space \parencite{connolly1997harmonic}. In previous work, we combined harmonic potential functions with the dynamical system framework to obtain reactive, local minima-free motion \parencite{huber2019avoidance, huber2023avoidance}. It allows for generating a desired velocity based on the robot's position in real time.
For implementation on a torque-controlled robot arm, we utilized a passive controller that closely adheres to the desired velocity \parencite{kronander2015passive}. However, one of the limitations of the passive controller is its inability to account for the physical surroundings. This makes the robot susceptible to disturbances close to obstacles, potentially leading to collisions.
Although DS passive-controlled robots work well in interactive scenarios, they cannot ensure they navigate through an obstacle environment collision-free. For example, they often give in when being pushed toward an obstacle. This work presents a method to address this problem by modifying the passive control law design, making the controller aware of its environment.
  
\subsection{Contribution}
We introduce a passive controller that incorporates into the feedback control loop as visualized in Figure~\ref{fig:control_scheme_passive}. 
In this work, we make the following contributions:
\begin{itemize}
\item A novel obstacle-aware passive controller
(Sec.~\ref{sec:obstacle_aware_passivity})
\item A passivity guarantee (without storage tank) which applies to general damping controllers (Theorem~\ref{theorem:passivity})
\ifthesis
{}
\else
\item A collision avoidance analysis which provides insight into the path consistency around obstacles (Sec.~\ref{sec:collision_avoidance})
\fi
\iflong
\item Discrete-time analysis to enable control parameter design which ensures stability (Section~\ref{sec:discrete_time_behavior})
\else
\fi
\item Implementation and testing on 7DoF robot arm (Sec.~\ref{sec:evaluation})
\end{itemize}

\begin{figure}[thb]
  \center
  \includegraphics[width=1.0\columnwidth]{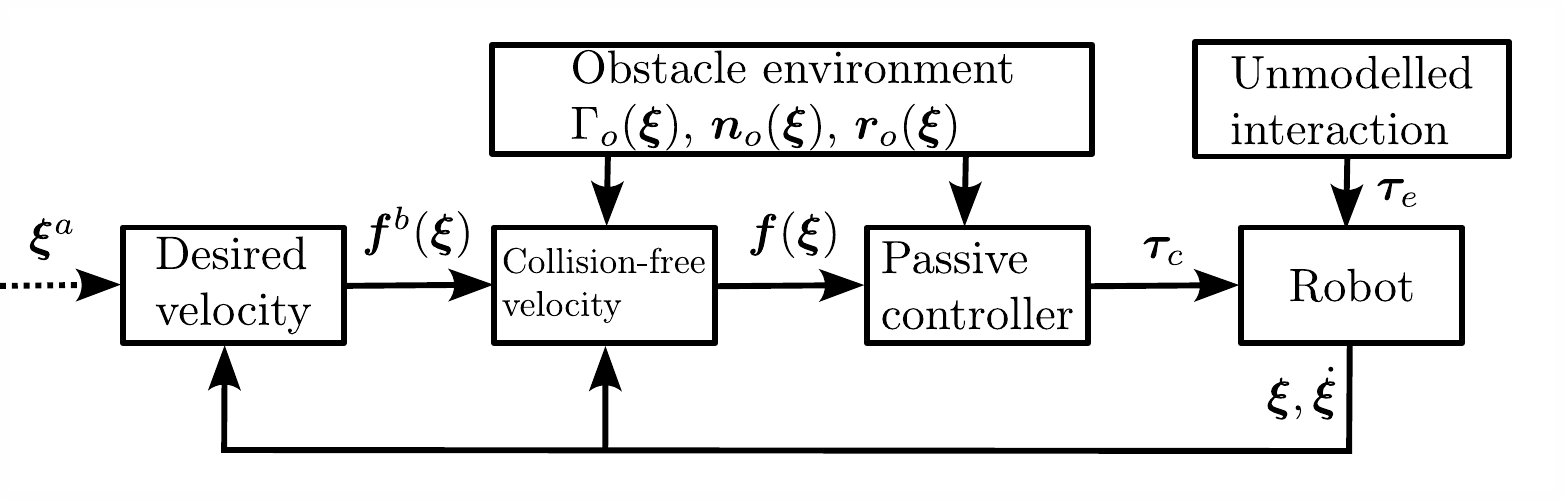}
\caption{The desired velocity $\vect f^b(\vecs \xi)$ can result from a learned velocity field or pointing towards a desired attractor $\vecs \xi^a$. The desired velocity is used to evaluate the obstacle avoidance velocity $\vect f(\vecs \xi)$, fed into the force controller to obtain the control force $\vect \tau_c$. In order to achieve collision avoidance, the distance function $\Gamma_o(\vecs \xi)$, the normal direction $\vect n_o(\vecs \xi)$, and the reference direction $\vect r_o(\vecs \xi)$ are evaluated for each obstacle $o = 1, ..., N^\mathrm{{obs}}$.}
\label{fig:control_scheme_passive}
\end{figure}

\section{Preliminaries}
Let $\vecs\xi \in \mathbb{R}^N$ describe the system's \editcolor{configuration} in an $N \geq 2$ dimensional space, e.g., the robot's joint or Cartesian space positions.
The function $\vecs f(\vecs \xi): \mathbb{R}^N \rightarrow \mathbb{R}^N$ represents a smoothly defined dynamical system (DS) describing the desired velocity at a given state $\vecs \xi$.  
The first and second-order time derivatives are denoted by one and two dots over the symbol respectively, i.e., $\dot{\vecs \xi} =\frac{d}{dt} \vecs \xi$ is the systems velocity, and $\ddot{\vecs \xi} = \frac{d^2}{dt^2} \vecs \xi$ is the acceleration.
In general, superscripts are used for variable names, whereas subscripts are used for enumerations.
\editcolor{The development of the presented work is for starshaped obstacles, i.e., shapes for which a reference point exists, from which a line in any direction only crosses the surface once \cite{huber2019avoidance}.}

\subsection{Obstacle Avoidance}
\editcolor{
This section is based on development proposed in \cite{huber2019avoidance, huber2022avoiding}, these articles contain more deteailled elaborations.
}

\subsubsection{Obstacle Definition}
\editcolor{
The distance function $\Gamma(\vecs \xi): \mathbb{R}^N \setminus \mathcal{X}^i \mapsto \mathbb{R}_{\geq 1}$ divides the space into three regions:
\begin{align}
  &\text{Exterior points:}&  \qquad & \mathcal{X}^e = \{\vecs \xi \in \mathbb{R}^N: \Gamma(\vecs \xi) > 1 \} \nonumber \\ 
  &\text{Boundary points:}&  \qquad & \mathcal{X}^b = \{\vecs \xi \in \mathbb{R}^N: \Gamma(\vecs \xi) = 1 \} \label{eq:levelFunc} \\
  &\text{Interior points:}&  \qquad & \mathcal{X}^i = \{ \vecs \xi \in \mathbb{R}^N \setminus (  \mathcal{X}^e \cup \mathcal{X}^b ) \} \nonumber
\end{align}
By construction $\Gamma(\cdot)$ increases monotonically with increasing distance from $\vecs \xi^r$. In this work, we use:
\begin{equation}
  \Gamma(\vecs \xi) = 1 + \| \vecs \xi - \vecs \xi^b \|  / R^0
  \quad \text{with} \quad
  \vecs \xi^b = b (\vecs \xi - \vecs \xi^r) + \vecs \xi^r
  \label{eq:distance_function}
\end{equation}
with $b \in \mathbb{R}_{>0}$, such that $\vect \xi^b \in \mathcal{X}^b$, i.e., on the surfacee, and $R^0 \in \mathbb{R}_{>0}$ is the distance scaling. We use $R^0 = 1$.
}

\subsubsection{Avoidance Algorithm}
Let us assume the base velocity $\vect f^b(\vecs \xi): \mathbb{R}^N \rightarrow \mathbb{R}^N$, which describes the desired, configuration-dependent motion of the robot. 
The obstacle avoiding velocity \editcolor{$\vect f(\vecs \xi): \mathcal{X}^e \rightarrow \mathcal{X}^e$} can be achieved by a simple matrix multiplication (or modulation) as follows:c
\editcolor{
\begin{equation}
\begin{split}
  & \vecs f(\vecs \xi) = \textbf{E}(\vecs \xi) \text{diag} \left(\lambda^r, \lambda^e, ..., \lambda^e \right) \textbf{E}(\vecs{\xi})^{-1} \vect f^b(\vecs \xi) \\
%
& \text{with} \quad
\textbf{E}(\vecs \xi) = \left[ \textbf{r}(\vecs \xi) \ \textbf{e}_1(\vecs \xi) \ ... \ \textbf{e}_{N-1}(\vecs \xi) \right]
\end{split}
  \label{eq:modulated_ds}
\end{equation}
}
where the tangent directions $\textbf{e}_{(\cdot)} \in \mathbb{R}^N$ are perpendicular to the surface normal \editcolor{$\vect n(\vecs \xi)  = \nabla \Gamma(\vecs \xi)  \in \mathbb{R}^N$}, see Fig.~\ref{fig:resultant_normal}. The reference vector $\textbf{r}(\vect{\xi}) =  \left( \vecs{\xi}-\vecs{\xi}^r \right) / \| \vecs \xi-\vecs \xi ^r \|$ is pointing \editcolor{away} from the reference point $\vecs \xi^r \in \mathbb{R}^N$. 

The diagonal values $\lambda_{(\cdot)}$ in \eqref{eq:modulated_ds} are often referred to as eigenvalues since they modify the length of the input velocity in specific directions. 
The eigenvalue in reference direction $\lambda^r(\vecs \xi) \leq 1$, is designed to reduce the velocity towards the obstacle.  
Conversely, the velocity increases along the tangent direction using $\lambda^e(\vecs \xi) \geq 1$. The eigenvalues in reference direction $\lambda^r$ and tangent direction $\lambda^e$ are defined as:
\begin{equation}
\begin{split}
    \lambda^r(\vecs \xi) = 1 - 1 /\Gamma(\vecs \xi) , \quad \lambda^e(\vecs \xi) = 1 + 1 / \Gamma(\vecs \xi)
    \label{eq:eigenvalues}
    \end{split}
\end{equation}
with the continuous distance function $\Gamma(\vecs \xi)$ given in \eqref{eq:distance_function}.

\begin{figure}[b]
\iflong
\centerline{\includegraphics[width=0.95\columnwidth]{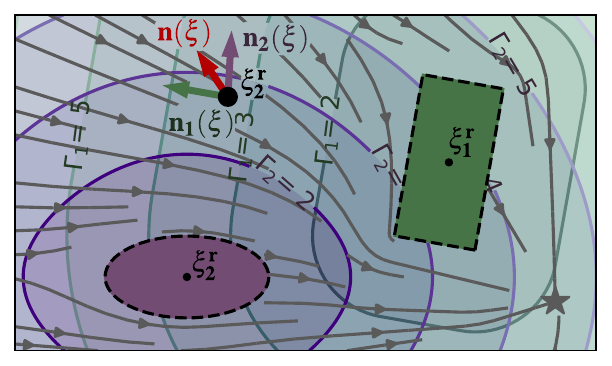}}
\else
\centerline{\includesvg[width=0.8\columnwidth]{figures/normal_and_gamma_field_visualization_annotated.svg}}
\fi
\caption{
The $\Gamma$-field is defined individually for each of the obstacles. At each position $\vecs \xi$, we can evaluate the surface normal $\vect n(\vecs \xi)$. 
The velocity $\vect f(\vecs \xi)$ (gray) avoids collision with the obstacles and converges towards the attractor (star).}
\label{fig:resultant_normal}
\end{figure}

\iflong
Modifying the base velocity $\vect f^b(\vecs \xi)$ with these eigenvalues values and \eqref{eq:modulated_ds} leads to obstacle-avoiding dynamics $\vect f(\vecs \xi)$ which generate converging motion around starshaped obstacles.
\fi

\subsection{Force Control}

\subsubsection{Rigid Body Dynamics}
A force-controlled system is subject to acceleration, inertia, and external disturbances. Its general rigid-body dynamics based on the state $\vecs \xi$: 
\begin{equation}
\matd{M}(\vecs\xi)\vecs{\ddot\xi} + \matd{C}(\vecs\xi, \vecs{\dot\xi})\vecs{\dot\xi} + \vect g(\vecs\xi) = \vecs{\tau_c} + \vecs{\tau_e}
 \label{eq:robot_dynamics}
\end{equation}
where we have the mass matrix of the robot $\matd M(\vecs\xi) \in \mathbb{R}^{N \times N}$, the Coriolis matrix $\matd C(\vecs\xi,\vecs{\dot\xi}) \in \mathbb{R}^N$, the gravity vector $\vect g(\vecs\xi) \in \mathbb{R}^N)$, the control torque $\vecs{\tau_c} \in \mathbb{R}^N$, and the external torque, also referred as disturbance, $\vecs{\tau_e} \in \mathbb{R}^N$.

\subsubsection{Damping Controller}
Damping control \cite{kronander2015passive} offers a powerful method for computing control forces from a velocity field. This controller provides selective disturbance rejection based on the direction of the desired motion. Typically, the controller is configured with high damping along the direction of motion, ensuring rapid convergence of the robot's velocity to the desired value and achieving excellent tracking performance. In contrast, the controller exhibits high compliance in the direction perpendicular to the motion, enabling flexible behavior and greater resistance to external forces. The passive control force is defined as
\begin{equation}
	\vecs{\tau_c} = \vect g (\vecs\xi) 
	+ \matd{D}(\vecs\xi) \left(\vecs f(\vecs\xi) - \vecs{\dot\xi} \right) 
\label{eq:control_command}
\end{equation}
This control law embeds a gravity compensation term $\vect g (\vecs\xi) \in \mathbb{R}^N$ and a positive-definite damping term, which dampens the difference between the desired velocity $\vecs f(\vecs\xi)$ and the actual velocity $\vecs{\dot\xi}$.
The positive definite damping matrix $\matd D(\vecs\xi) \in \mathbb{R}^{N \times N}$ is given as:
\begin{equation}
   \matd {D}(\vecs \xi) = \matd{Q}(\vecs\xi)\matd{S}(\vecs\xi) \matd{Q} (\vecs \xi)^{-1}
\label{eq:damping_matrix}
\end{equation}
where $\matd Q(\vecs \xi) = \left[ \vecs{q}_1 \;  \vecs q_2  \; ... \; \vecs q_N \right] $ is an orthonormal basis matrix, of which the first vector is pointing in the desired direction of motion
\begin{equation}
    \vecs q_1(\vecs \xi) = \vecs q_1^f(\vecs \xi) = \vect f({\vecs \xi}) / \lVert \vect f({\vecs \xi}) \rVert \label{eq:velocity_unit_vector}
\end{equation}

The diagonal matrix $\matd{S}(\vecs\xi) \in \mathbb{R}^{N \times N}$ consists of damping factors $s_i \in \mathbb{R}_{>0}$ in directions $i = 1, ... , N$.
This design allows the separate design of the damping in the direction of motion and perpendicular to the motion.
Increased consistency with the desired velocity is achieved through a high value for the first damping factor. Conversely, the damping in the remaining directions is lower for more compliance perpendicular to the motion, i.e., $s_i / s_1 \ll 1, \; i = 2,  ... , N$.

\iflong
\subsection{Stability Analysis} \label{sec:trad_passive}
Considering varying control parameters carefully is crucial, as such a design can inject energy into a system, potentially leading to unstable behavior and damaging the system \cite{ferraguti2013tank}.
In human-robot interaction, the robot faces external disturbances of unknown nature. To achieve stable and bounded behavior in the face of such disturbances, \textit{passivity} analysis is a valuable tool. By employing passivity analysis, the control system can be designed to maintain stable responses (bounded system) in the presence of any external force (bounded input). 

\begin{definition}[Passivity \cite{willems1972dissipative, sepulchre2012constructive}] \label{def:passivity}
	A dynamical system with input $ u \in \mathcal{U}$ and output $y \in \mathcal{Y}$ is passive with respect to the supply rate $s : \mathcal{U} \times \mathcal{Y} \rightarrow{R}$ if, for any $u: \mathbb{R}_{>0} \rightarrow \mathcal{U}$ and any time $t^* \geq 0$ the following is satisfied
  \begin{equation}
	  \int_0^{t^*} s \left( u(t),  y (t) \right) d \tau \geq S(t^*) - S(0) 
  \end{equation}
  where $S(t) \in \mathbb{R}_{\geq 0}$ is the storage function.
\end{definition}

The feedback loop combining two passive systems results in a passive system \cite{sepulchre2012constructive}. Hence, if the controller is passive, its application to a (passive) environment-hardware system results in a passive system.
\fi

\iflong
\subsection{Problem Statement}
Following assumptions are made about the desired velocity $\vecs f(\vecs\xi)$:
\begin{enumerate}
    \item $\vecs f(\vecs\xi)$ is continuous for all reachable states.
    \item $\vecs f(\vecs\xi)$ is bounded, i.e., there exists a constant $v^{\mathrm{max}} \in \mathbb{R}$ such that $\| \vecs f(\vecs\xi) \| \leq v^{\mathrm{max}} \;\; \forall \, \vecs \xi \in \mathbb{R}^N$
    \item $\vecs f(\vecs\xi)$ leads to a collision-free motion, i.e., $\vecs{n_o}(\vecs\xi)^T \vecs f(\vecs\xi) \geq 0$ as $\Gamma_o(\vecs \xi) \rightarrow 1 \quad \forall o = 1, ..., N^{\mathrm{o}}$ with the normal $\vecs n_o$ and distance $\Gamma_o$ of the $o$-th obstacle. 
\end{enumerate}

Note that velocity obtained using the obstacle avoidance method described in \eqref{eq:modulated_ds}, fulfills these conditions if base velocity $\vecs f^b(\vecs \xi)$ is continuous and bounded.

\fi

\section{Obstacle-Aware Passivity} \label{sec:obstacle_aware_passivity}
We propose a novel controller, which ensures passivity as defined in \eqref{eq:control_command} but adapts the damping matrix given in \eqref{eq:damping_matrix} based on the desired velocity $\dot{\vecs \xi}$ and obstacles in the surrounding. 
Far away from obstacles, the system is designed to follow the initial velocity, but approaching the obstacle increases the damping, decreasing the chance of a collision.

\ifthesis
\begin{figure}
\centerline{\includegraphics[width=0.7\columnwidth]{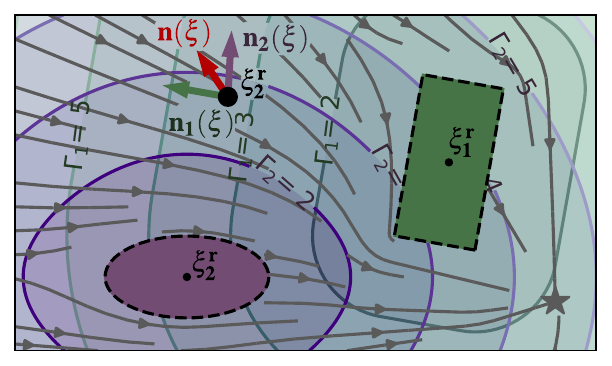}}
\caption{
The $\Gamma$-field is defined individually for each of the obstacles. At each position $\vecs \xi$, we can evaluate the surface normal $\vect n(\vecs \xi)$. 
The velocity $\vect f(\vecs \xi)$ (gray) avoids collision with the obstacles and converges towards the attractor (star).}
\label{fig:resultant_normal}
\end{figure}
\fi

Hence, the damping matrix $\matd D(\vecs\xi) \in \mathbb{R}^{N \times N}$ smoothly changes from being aligned with the direction of the velocity, as used in \parencite{kronander2015passive}, to be aligned with the averaged normal of the obstacles. The desired damping matrix transitions between velocity preserving and collision avoidance using a smoothly defined linear combination:
\editcolor{
\begin{equation}
	\matd D(\vecs\xi, \dot{\vecs \xi}) = \left(1 - w(\vecs\xi) \right) {\matd D^{f}}(\vecs\xi) + w(\vecs\xi)  {\matd D^{\mathrm{o}}}(\vecs\xi, \dot{\vecs \xi}) \label{eq:damping_summation}
\end{equation}
}
The damping matrix is made up of two components: the velocity damping. $\matd D^f \in \mathbb{R}^{N \times N}$ which prioritizes following the desired velocity similar to \parencite{kronander2015passive}, and the obstacle damping $\matd D^{\mathrm{o}} \in \mathbb{R}^{N \times N}$ which is designed to reject disturbances towards obstacles. The two damping matrices are positive definite and are smoothly summed using the danger weight $w(\vecs\xi) \in [0, 1]$. Far away from obstacles the weight reaches $w(\vecs \xi) = 0$, whereas $w(\vecs \xi) = 1$ when approaching a boundary:
\begin{equation}
  \begin{split}
w(\vecs\xi) =
\max \left(0,  \frac{\Gamma^{\mathrm{crit}} - \Gamma(\vecs\xi)}{\Gamma^{\mathrm{crit}} - 1} \right) \| \vecs n(\vecs \xi) \| \\
\text{with} \quad
\Gamma(\vecs\xi) = \min_{o = 1, ..., N^{\mathrm{obs}}} \Gamma_o(\vecs\xi)
\end{split}
\label{eq:weight_function}
\end{equation}
The critical distance \iflong $\Gamma^{\mathrm{crit}} \in \mathbb{R}_{>0}$ \fi defines the distance where the system increases the damping towards the obstacle.
${\matd D^f}(\vecs\xi)$ and $\matd {D^{\mathrm{obs}}}(\vecs\xi)$ follow design given in \eqref{eq:damping_matrix} and are positive definite matrices, thus $\matd {D}(\vecs\xi)$ is positive definite, too.

\iflong
\begin{figure}
  \center
  \ifthesis
  \includegraphics[width=0.5\columnwidth]{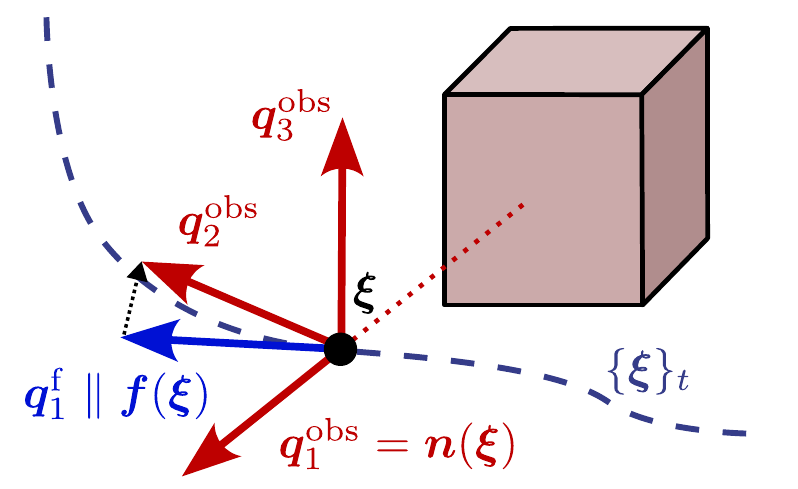}
  \else
  \includegraphics[width=1\columnwidth]{figures/damping_basis_construction.pdf}
  \fi
\caption{The damping matrix enforcing desired velocity following $\matd{D}^{f}$ has the first basis vector $\vect q_1^{f}$ which points along the avoidance velocity $\vect f(\vecs \xi)$. The damping matrix to enforce collision avoidance $\matd{D}^{\mathrm{obs}}$ uses the normal $\vect n(\vecs \xi)$ to construct the first direction of the decomposition basis $\vect q^{\mathrm{obs}}_1$.}
\label{fig:damping_basis_construction}
\end{figure}
\fi

\subsection{Damping for Collision Repulsion} \label{sec:obstacle_repulsion}

\subsubsection{Normal Direction}
The damping matrix $\matd D^{\mathrm{o}}(\vecs \xi)$ rejects velocities in the direction of the obstacles. To allow this, we introduce an averaged normal direction:
\begin{equation}
  \vecs n(\vecs\xi) = \sum_{o=1}^{N^{\mathrm{o}}} \vecs{n_o}(\vecs\xi)
  \frac{1 / (\Gamma_o(\vecs \xi) - 1)}{\sum_{p=1}^{N^\mathrm{o}} 1 / (\Gamma_p(\vecs \xi) - 1)}
  \label{eq:averaged_normal}
\end{equation}
 where the unit normals $\vecs{n_o} (\vecs\xi)$ pointing away from the obstacle $o = 1,  ..,  N^{\mathrm{o}}$ and are perpendicular to the surface, see Figure~\ref{fig:resultant_normal}. 

The averaged normal $\vecs n(\vecs \xi)$ is a weighted linear combination of the obstacles' normals, giving more importance to closer obstacles.
Additionally, the averaged normal converges to an obstacle normal as we converge towards it, i.e., $\lim_{\Gamma_o(\vecs \xi) \rightarrow 1} \vecs n(\vecs \xi) = \vecs n_o(\vecs \xi)$.
Note that the averaged normal is a zero-vector when two obstacles oppose each other. 

\subsubsection{Decomposition Matrix}
The decomposition matrix $\matd Q^{\mathrm{o}}(\vecs \xi)$ has its first vector aligned with the normal to the obstacle: 
\begin{equation}
    \vecs q_1^{\mathrm{o}}(\vecs\xi) =  \vecs n(\vecs\xi) / \lVert\vecs n(\vecs\xi)\rVert 
    \quad \forall \, \vecs \xi : \lVert\vecs n(\vecs\xi)\rVert  > 0
    \label{eq:first_obstacle_basis}
\end{equation}

In the case that $\lVert\vecs n(\vecs\xi)\rVert = 0$, the danger weight $w(\vecs \xi)$ from \eqref{eq:weight_function} reaches 0. Hence, the obstacle-aware damping in \eqref{eq:damping_summation} has no effect and is not evaluated.

The second vector is set to be aligned with the desired velocity as much as possible, allowing increased velocity following \iflong (Fig.~\ref{fig:damping_basis_construction}) \else (Fig.~\ref{fig:resultant_normal})\fi. However, it has to remain orthonormal to $\vect q_1^{\mathrm{o}}$
\begin{equation}
  \vecs q_2^{\mathrm{o}} = \frac{\hat{\vecs q}_2^{\mathrm{o}}}{\| \hat{\vecs q}_2^{\mathrm{o}} \|}
  \quad
  \hat{\vecs q}_2^{\mathrm{o}} = \vecs q_1^{f} - \vecs q_1^{\mathrm{o}} p \quad  \forall \, \vecs \xi : | p | < 1
\end{equation}
where velocity unit vector $\vecs q_1^{f}$ is defined in \eqref{eq:velocity_unit_vector}, and the object weight is evaluated as $p = \dotprod{\vecs q_1^{\mathrm{o}}}{\vecs q_1^{f}}$. 
For the case that $| p | = 1$, the second basis $\vecs q_2^{\mathrm{o}}$ is set to be any orthonormal vector. The remaining vectors $\vecs q_d^{\mathrm{o}}, d = 3, ..., N$ are constructed to form an orthonormal basis to the first two.

\subsubsection{Damping Values}
We define the values of the diagonal matrix $\matd{S}^{\mathrm{o}}(\vecs \xi)$ as\begin{equation}
  \matd{S}_d^{\mathrm{o}}(\vecs \xi) =
  \begin{cases}
    s^{\mathrm{o}} & d = 1 \\
    | p | s^c + (1 - | p |) s^{f} & d = 2 \\
    s^c & d = 3, ..., N
  \end{cases}
  \label{eq:obstacle_damping_values}
\end{equation}
where the damping along the nominal direction $s^{f} \in \mathbb{R}_{>0}$, obstacle-damping $s^{\mathrm{o}} \in \mathbb{R}_{>0}$, and the compliant-damping $s^c \in \mathbb{R}_{>0}$ are user-defined parameters which determine the behavior of the passive-controller. The first entry of $\matr{S}^{\mathrm{o}}$ dictates the damping towards the obstacle, and the second entry the desired velocity following. Note how the second value approaches the compliant damping, as normal and velocity become parallel.

To ensure continuity across time of the control force as defined in \eqref{eq:control_command}, the values of the diagonal damping matrix  $\matd{S}^{\mathrm{o}}(\vecs \xi)$ in the tangent directions are equal when the normal aligns with the velocity:
\begin{equation}
    \lim_{| p | \rightarrow 1} \matd{S}_d = \matd{S}_e, 
    \quad d > 2, ..., N, \, e > 2, ..., N
\end{equation}
Hence, the choice of orthonormal vectors $\vecs q_d^{\mathrm{o}}, d > 2, ..., N$ does not influence the control force as long as the matrix $Q^{\mathrm{d}}$ has full rank.

\subsubsection{Damping Only Towards Obstacle} \label{sec:damping_only_toward}
In the presence of an obstacle, the disturbances should be damped strongly when the agent is pushed against the obstacle. Conversely, the system can remain compliant if the motion is away from the obstacle. This is achieved by setting updating the first damping value $\matd{S}^{\mathrm{o}}_1$ if the robot is moving away from the surface: 
\editcolor{
\begin{equation}
	\matd{S}_1^{\mathrm{o}} (\vecs \xi, \dot{\vecs \xi}) =
  \begin{cases}
    s^{\mathrm{o}} & \text{if} \;\; \left(\vect f(\vecs \xi) - \vecs{\dot \xi} \right)^T \vecs n(\vect \xi) > 0 \\
    s^{\mathrm{c}} & \text{otherwise}
  \end{cases}
  \label{eq:leaving_compliance}
\end{equation}
}

Since $\vect q_1^o(\vecs \xi)$ given in \eqref{eq:first_obstacle_basis} is pointing along the obstacle normal $\vect n(\vecs \xi)$, the first obstacle damping value $\matd{S}_1^{\mathrm{o}} (\vecs \xi)$ does not have any effect on the control force $\vect \tau^c$ given in \eqref{eq:control_command} when $\left(\vect f(\vecs \xi) -  \vecs{\dot \xi} \right)^T \vecs n(\vect \xi) = 0$.
Hence, the damping value can be discontinuous across time, but the resulting control force remains continuous.

\subsection{Damping for Velocity Preservation}
The decomposition matrix $\matd{Q}^{f}$ is an orthonormal basis with the first vector being parallel to the desired velocity $\vect f(\vecs\xi)$\iflong, as proposed in Section~\ref{sec:trad_passive}\fi. Hence, the values of the diagonal matrix $\matd S^{f}$ are high in the direction of the desired velocity (first component) but more compliant in the remaining directions. 
Moreover, the damping is set to ensure that when passing a narrow passage between two obstacles, where the normal vector cancels $\vect n(\vecs \xi) \approx \vect 0$, with additionally a low distance value $\Gamma(\vecs \xi) \approx 1$, the damping perpendicular to the velocity direction is high. We set:
\begin{equation}
  \begin{split}
  & \matd{S}^{f}_{d} =
  w^p s^{\mathrm{o}} + 
  \begin{cases}
   (1- w^p) s^f & d = 1 \\
   (1- w^p) s^s & d = 2, ..., N 
  \end{cases} \\
  & \text{with} \quad
   w^p = \min \left(1,  \| \vecs n(\vecs \xi) \|^2 +  \Delta \Gamma ^2 \right) \\
   & \text{and} \quad \Delta \Gamma = \max \left(\frac{\Gamma^{\mathrm{crit}} -\Gamma(\vecs \xi)}{\Gamma^{\mathrm{crit}} - 1}, 0\right)
  \end{split}
  \label{eq:velocity_damping}
\end{equation}

\subsection{Cluttered Environments}
In a cluttered environment, the normal vectors of the individual obstacles can be opposing. And hence using \eqref{eq:averaged_normal} and \eqref{eq:weight_function}, we get: 
\begin{equation}
    \| \vecs n(\vecs \xi) \| \rightarrow 0
    \quad \text{and} \quad
    \lim_{\Gamma (\vecs \xi) \rightarrow 1} w (\vecs \xi) = 0
\end{equation}
Additionally using \eqref{eq:damping_summation} and \eqref{eq:velocity_damping} we obtain:
\begin{equation}
    \lim_{\Gamma \rightarrow 1, \| n \| \rightarrow 0} \matd{D}(\vecs \xi) 
    = \matd{D}^S(\vecs \xi) + 0 
    =  \matd{I} s^o
\end{equation}

Hence, there is high damping in all directions to reject disturbances towards potential obstacles and ensure a collision-free motion even in cluttered environments.

\iflong
\subsection{Damping Parameter Design}
Higher values for the damping parameters $s^{(\cdot)}$ generally result in improved velocity following and disturbance repulsion, whereas lower values allow more compliant behavior.
The damping value in the direction of the obstacle is set high $s^{\mathrm{o}}$ to ensure obstacle avoidance even in the presence of high disturbances. 
Conversely, the damping in the direction of the velocity $s^{f}$ is set medium to high, as the system should follow the desired velocity $\vect f(\vecs \xi)$. However, it should remain compliant if needed.
Finally, to facilitate interaction, a low damping value $s^{c}$ should be chosen for all other directions.
This can be summarized as:
\begin{equation}
s^{\mathrm{o}} \geq s^{f} \gg s^{c} > 0
\end{equation}
\fi

\subsection{Passivity Analysis}
The stability analysis of the system gives information about the stability region of the proposed controller. We analyze passivity by observing the evolution of the kinetic energy of the system, given as:
\begin{equation}
	W(\vecs \xi, \vecs{\dot \xi}) = \frac{1}{2}  \dot{\vecs{\xi}}^T \matd{M}(\vecs \xi) \dot{\vecs{\xi}} \label{eq:energy_system}
\end{equation}

\begin{lemma} \label{lemma:passivity}
	Let us assume a robotic system as described in \eqref{eq:robot_dynamics} is controlled using \eqref{eq:control_command} using the damping matrix \editcolor{$\matd D(\vecs \xi, \dot{\vecs \xi})$} given in \eqref{eq:damping_summation} with damping values $s_d = 1, d = 1, ..., N$.
   The system is passive with respect to the input-output pair $\vecs \xi_e$, $\vecs{\dot \xi}$ when exceeding the desired velocity $\vect f(\vecs \xi)$ , i.e., $\dot{W} \leq \vecs{\dot \xi}^T \vecs \tau^e, \; \forall \vecs \xi \in \mathbb{R}^N: \| \vecs{\dot \xi} \| \geq \| \vect f(\vecs \xi) \|$ and the storage function being the kinetic energy $W \in \mathbb{R}_{>0}$ given in \eqref{eq:energy_system}
\end{lemma}

\iflong

\begin{proof}
The time derivative of storage function $W$ can be evaluated as:
\editcolor{
\begin{align}
	& \dot W(\vecs \xi, \vecs{\dot \xi}) =
	\vecs{\dot \xi}^T \matd M(\cdot) \vecs{\ddot \xi}  + \frac{1}{2} \vecs{\dot \xi}^T \dot{\matd M}(\cdot) \vecs{\dot \xi}  \nonumber \\
	&= \frac{1}{2} \vecs{\dot \xi}^T \left( \dot{\matd M}(\cdot) - 2 \matd C(\cdot) \right) \dot{\vecs \xi}  
	 \quad - \vecs{\dot \xi}^T \matd{D}(\cdot) \left(\vecs{\dot \xi} - \vect f(\vecs \xi) \right) + \vecs{\dot \xi}^T \vecs \tau^e \nonumber \\
  &= - \vecs{\dot \xi}^T \matd{D}(\cdot) \left( \vecs{\dot\xi} - \mathbf{f}(\vecs \xi) \right) + \vecs{\dot\xi}^T \vecs{\tau}^e
\end{align}
}
where the second order dynamics $\vecs{\ddot \xi}$ are evaluated according to the rigid body dynamics defined in \eqref{eq:robot_dynamics}. Furthermore, $\dot{\matd M} - 2 \matd C$ is skew-symmetric for any physical system \parencite{slotine1987adaptive}; hence, the corresponding summand is zero.

Let us investigate the region where the passivity holds. Since in the Lemma, we assumed all damping values to be equal to one, we have:
\editcolor{
\begin{equation}
	\matd{D}({\vecs \xi}, \dot{\vecs \xi}) = \matd{Q}(\cdot) \matd{S} (\cdot) \matd{Q}(\cdot)^{-1}= \matd{Q}(\cdot) \matd{I} \matd{Q}(\cdot)^{-1} = \matd{I}
\end{equation}
}
where $\matd{I} \in \mathbb{R}^{N \times N}$ is the identity matrix.

It follows that the system is passive with respect to the input, the external force $\tau^e$, and the output, the velocity $\dot {\vecs \xi}$, if: 
\begin{equation}
	\dot{\boldsymbol {\vecs \xi}}^T \matd{D}({\vecs \xi}) \left(\dot{\boldsymbol {\vecs \xi}} - \boldsymbol{f}(\boldsymbol {\vecs \xi}) \right) = 
    \dot{{\vecs \xi}} ^ T \Delta \vect{f}  \geq 0 
 \; , \quad
 \Delta \vect{f} = \dot{{\vecs \xi}} - \vect{f}({\vecs \xi})
 \label{eq:passivity_condition}
\end{equation}

On the border of this region, the two vectors $\Delta \vect{f}$ and $\dot{{\vecs \xi}}$ are orthogonal.
Hence, using Thale's theorem, this region can be interpreted as a circle in velocity-space with radius $\| \vect{f} ({\vecs \xi}) \| / 2$ and center $\vect{f}({\vecs \xi}) / 2$, see Figure~\ref{fig:passivity_analysis}.

\begin{figure}[bh]
	\centering
	\ifthesis
    \includegraphics[width=0.7\columnwidth]{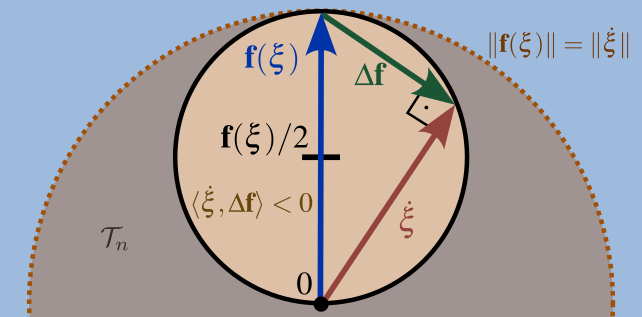}
	\else
    \includegraphics[width=.95\columnwidth]{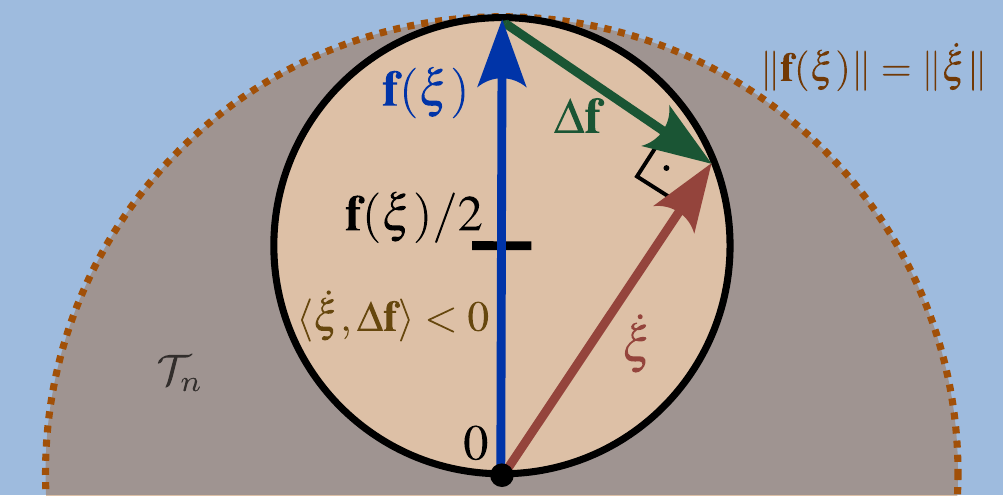}
	\fi
	\caption{Analyzing the system in velocity-space yields that the system is passive if it has a velocity $\dot{\vecs \xi}$ larger than the desired velocity $\vect f(\vecs \xi)$, i.e., outside the dashed circle.
    However, the system can be non-passive for small velocities when  $\dotprod{\dot{\vecs \xi}}{\Delta \vect f} < 0$ (yellow circle).}
	\label{fig:passivity_analysis}
\end{figure}

Moreover, the system is passive as long as the observed velocity $\dot{{\vecs \xi}}$ is outside the circular-red region, which is a subset of the region where the magnitude of the observed velocity is smaller than the desired velocity $\vect {f}({\vecs \xi})$, i.e.,
\begin{equation}
	\dot W({\vecs \xi}, \dot{{\vecs \xi}}) \leq \dot{{\vecs \xi}}^T \vecs \tau^e
 \quad \forall {\vecs \xi} : \| \dot{{\vecs \xi}} \| \geq\| \vect{f}({\vecs \xi}) \| 
\end{equation}

\end{proof}

As in the orange region, the system is not passive; the storage function $W$ could increase, and hence, the velocity $\dot {\vecs \xi}$ increases non-passively. This behavior is not unexpected, as the controller is designed to approach the desired dynamics $\vect{f}({\vecs \xi})$. Hence, as long as the desired velocity is not reached, the kinetic energy increases even with no force input $\vecs \tau^e$. However, as soon as the system velocity $\dot{\vecs \xi}$ exceeds the desired velocity $\vect f(\vecs \xi)$, the system behaves passively. We can use this to ensure the stability of the system:

\else
\editcolor{
\begin{proof}
See extended artice \parencite{passive2024huber_extended}. \hfill 
\end{proof}
}
\fi

\begin{theorem}  \label{theorem:passivity}
  Let $\vect f(\vecs \xi)$ is the desired velocity with bounded magnitude, i.e., $\| \vect f(\vecs \xi) \| < v^{\mathrm{max}}, \forall \vecs \xi \in \mathbb{R}^N$.
  The closed loop system \eqref{eq:robot_dynamics} using the controller from \eqref{eq:control_command} and the damping matrix $\matd D(\vecs \xi, \dot{\vecs \xi})$ given in \eqref{eq:damping_summation} is bounded-input, bounded-output (BIBO) stable with respect to the input disturbance force $\vect \tau^e$, and output the velocity $\dot{\vecs \xi}$ for all times $T = 0, ...,  \infty$.
\end{theorem}

\iflong
\begin{proof}
To ensure BIBO stability, let us analyze the integral of the impulse of the response for the external force $\vecs \tau^e$: 
\begin{equation}
	\begin{split}
	  \int_{0}^{T} \left\| \dot{\vecs \xi} \right\| \; dt 
	  & = \int_{t \notin \mathcal{T}_n} \left\| \dot {\vecs \xi} \right\|  \, dt + \int_{t \in  \mathcal{T}_n} \left\| \dot {\vecs \xi} \right\| \;  dt \\ 
	  & \leq K_p + v^{\mathrm{max}} T_n
\end{split}
\label{eq:bibo_velocity}
\end{equation}
where $\mathcal{T}_n$ denotes the set of time instances where the system is not shown to be passive (Fig.~\ref{fig:passivity_analysis}), specifically $\| \dot{\vecs \xi} \| \leq \| \vecs f (\vecs \xi) \|$, and $T_n \in \mathbb{R}_{\geq 0}$ is the total duration which the system spends in this region. Additionally, as ar result from passivity in the inner region, there exists a finite valued constant $K_p \in \mathbb{R}_{\geq 0}$ which marks the upper bound of the system.
Hence, the impulse response is bounded, and the system is BIBO stable.

However, from \eqref{eq:damping_summation}, we know that a general damping matrix $\matd{S}(\vecs \xi, \dot{\vecs \xi})$ can have non-uniform diagonal values. This is analyzed by introducing the coordinate transfers:
\begin{equation}
	\vecs{\bar{v}} = \sqrt{\matd{S}(\cdot)} \matd{Q}(\cdot)^{-1} \dot{{\vecs \xi}}
	\;\; \text{and} \;\;
	\bar{\Delta \vect f} = \sqrt{\matd{S}(\cdot)} \matd{Q}(\cdot)^{-1} \Delta \vect{f}
\end{equation}
where the square root of the diagonal matrix $\matd{S}({\vecs \xi})$ is taken element-wise.
The transfer is then used to rewrite \eqref{eq:passivity_condition} as:
\begin{equation}
	\dot{\vecs \xi}^T \matd{D}({\vecs \xi}, \dot{\vecs \xi}) \Delta \vect{f} = \vecs{\dot \xi}^T \matd{Q}(\cdot) \matd{S}(\cdot) \matd{Q}(\cdot)^{-1} \Delta \vect{f} = \vecs{\bar v}^T \bar{\Delta \vect f}
\end{equation}

Hence, the BIBO analysis of \eqref{eq:bibo_velocity} applied to the transformed system results as:
\begin{equation}
\begin{split}
	  & \int_{0}^{T} \left\| \vecs{\bar v} \right\| \; dt   
	   = \int_{t \notin \bar{\mathcal{T}}_n} \left\| \vecs{\bar v} \right\|  \, dt + \int_{t \in  \bar{\mathcal{T}}_n} \left\| \vecs{\bar v} \right\| \;  dt  \\ 
   & < K_p + v^{\mathrm{max}} \bar T_n 
   {\max{\Bigl(\text{eig}\bigl(\mathcal{D} \bigr) \Bigr)}} 
   / {\min{\Bigl(\text{eig}\bigl(\mathcal{D} \bigr) \Bigr)}}
\end{split}
\end{equation}
where $\bar{\mathcal{T}}_n$ denotes the region where the transformed system $\vecs{\bar v}$ is not shown to be passive, i.e. $\| \vecs{\bar v} \| \leq \| \bar{\Delta \vect f} \|$, and $\bar T_n \in \mathbb{R}_{\geq 0}$ the corresponding time. Additionally, $\min{(\text{eig}(\mathcal{D}))}$ and $\max{(\text{eig}(\mathcal{D}))}$ are the smallest and largest eigenvalue of the damping matrix $\matd{D}$ respectively.

Hence, since the transformed system with velocity $\vect {\bar v}$ is BIBO stable, the original system with velocity $\dot{\vecs \xi}$ is BIBO stable, too, as long as it is a continuous, finite transform. 
\end{proof}
\else
\editcolor{
\begin{proof} See extended article \parencite{passive2024huber_extended}. \end{proof}
}
\fi
For an orthogonal decomposition matrix $\matd{Q}(\boldsymbol{{\vecs \xi}}, \dot{\vecs \xi})$, the region of non-passivity is an ellipse where the direction of the axes points along column vectors of $\matd{Q}$, and the corresponding axes lengths are the diagonal elements of $\| \mathbf{f}({\vecs \xi})^T \sqrt{\matd{S}(\cdot)}\| / 2 \sqrt{\matd{S}(\cdot)}^{(-1)}$. 
If the ratio of the first damping value to the other axis $i \geq 2$ is large, i.e., $s_1 / s_i \gg 1$, it can lead to non-passivity even though the velocity $\dot{\vecs \xi}$ is already larger (but not pointing in the correct direction) than the desired velocity. However, the non-passive region is still \iflong bounded (Fig.\ref{fig:passivity_analysis_skew}). \else bounded. \fi
This proof holds for any basis $\matd{Q}$ which is not singular. However, the controller must be carefully chosen to ensure that the speed up is limited when the basis is close to singular, for example, by limiting the relative difference of the stretching vectors. Furthermore, as stable behavior is ensured for a general shape of a damping matrix $\matd{D}(\vecs \xi)$, the global stability proof extends to any positive definite damping matrix matrices.

Since the damping matrix $\matd D(\vecs \xi)$ changes dynamically, a change in the environment can move the system outside of the passive region. However, a finite maximum velocity always exists, at which the system is ensured to be passive.

\iflong
\begin{figure}[htbp]
    \centering
    \begin{subfigure}{0.49\columnwidth}
      \centerline{\includegraphics[width=\textwidth]{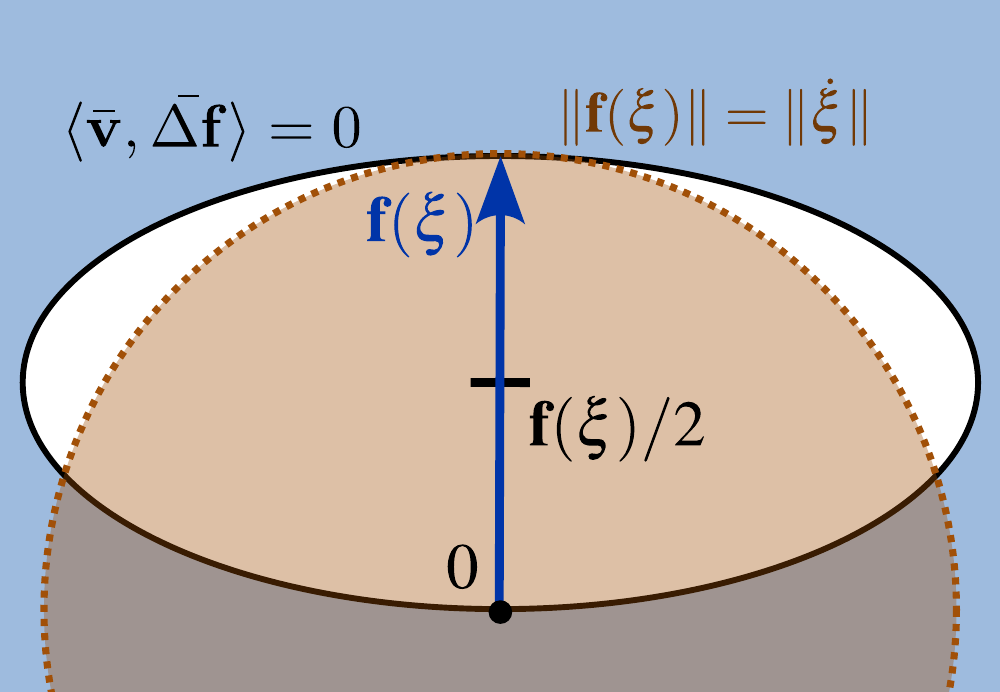}}
	  \caption{$\matd{S}_1^f > \matd{S}_2^f$, $w \approx 0$}
	  \label{fig:passivity_analysis_wide}
    \end{subfigure}\hfill%
    \begin{subfigure}{0.49\columnwidth}
    \includegraphics[width=\textwidth]{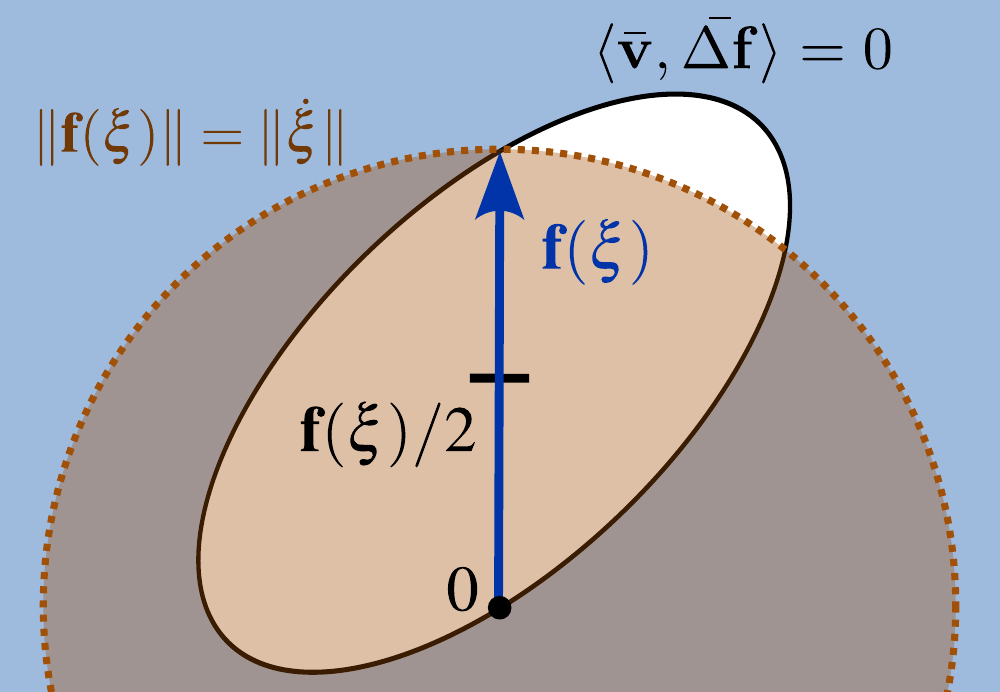}
	\caption{$\dotprod{\vect f(\vecs \xi)}{\vect q_1} \neq \| \vect f(\vecs \xi)\| \; \| \vect q_2\| $ }
      \label{fig:passivity_analysis_skew}
    \end{subfigure}
	\caption{
		The stability is ensured even if the controller can temporarily accelerate the system to reach a velocity $\dot{\vecs \xi}$, which is faster than the desired velocity $\vect f(\vecs \xi)$ (white region).
		This happens when the eigenvalues of the damping matrix $\matd{D}(\vecs \xi)$ are not uniform (a) or the stretching vectors $\vect q_1$ and $\vect q_2$ are not orthogonal (b). 
	In both cases, the region of non-passivity is elliptical (black circle).
}
	\label{fig:passivity_analysis_varied}
\end{figure}
\fi


\section{Collision Avoidance} \label{sec:collision_avoidance}

The principal goal of the controller introduced in the previous section is its ability to ensure collision avoidance in the presence of external disturbances.
However, the control force $\vect \tau^c$ proposed in \eqref{eq:control_command} does not explicitly consider external forces. Yet, it is designed to correct the agent's velocity $\dot{\vecs \xi}$ if it deviates from the desired velocity $\ddot{\vecs \xi}$.\footnote{Note that for a discrete-time (digital) controller, this results in a delay.}

Since interaction with the environment results in a force on the system, often over a short period $\Delta t \ll 1$.  Hence, we can define the velocity after impact $\vect v^I$ as:
\begin{align}
	\vect v^I
	  \approx \int_{t^I}^{t^I + \Delta t} \ddot{\vecs \xi} dt  
	  \approx \int_{t^I}^{t^I + \Delta t} \matd{M}^{-1}(\vecs \xi)  \vecs \tau_e \, dt  
	  \label{eq:impact_velocity}
\end{align}
using the controller from \eqref{eq:robot_dynamics}, and under absence of the control force $\vecs \tau_c$ during this short timeframe. Additionally, $\{\vect \xi \}_{t^I}$ is the velocity before the impact.
 


Furthermore, let us consider a desired velocity $\vect f(\vecs \xi)$, which is a constant, collision-free vector field parallel to the surface of a flat obstacle surface (see Fig.~\ref{fig:disturbance_with_parallel_velocity}):
\begin{equation}
	\dotprod{\vect f(\vecs \xi)}{\vecs n(\vecs \xi)} = 0
	 \qquad
\vect f(\vecs \xi) = \text{const.}
\, , \;
\vect n(\vecs \xi) = \text{const.}
\label{eq:parallel_velocity}
\end{equation}
where $\vecs n(\vecs \xi)$ is the surface's normal vector, and the agent moves in a straight line, hence we can neglect the Coriolis effect. For disturbances in such environments, we show that our approach can ensure the impenetrability of the obstacle up to an upper bound on the magnitude of the disturbance:

\begin{lemma} \label{lemma:damping_collision_avoidance}
	Consider a point-mass agent with mass $m \in \mathbb{R}_{>0}$, whose motion evolves according to the rigid body dynamics given in \eqref{eq:robot_dynamics} controlled by \eqref{eq:control_command}, with constant damping matrix $\matd{D}$ from \eqref{eq:damping_summation}. The agent tracks a constant reference velocity ${\mathrm{f}}$, whose vector field moves parallel to a flat obstacle as given in \eqref{eq:parallel_velocity}. Any motion path initiated in free space will remain collision-free for all times, i.e., $\Gamma( \{\vecs \xi_t\}) \geq 1$ with $t \geq 0$ if the impact velocity $v^I$ as given in \eqref{eq:impact_velocity} at time $t=0$ is limited by $\| \vect v^I\| < s^{\mathrm{f}} \| \vecs \xi - \vecs \xi^b \| / m$, with respect to the closest surface point $\vecs \xi^b \in \mathbb{R}^N$.
\end{lemma}


\iflong
\begin{proof}
According to the Bony-Bezis theorem \parencite{bony1969principe}, the trajectories are collision-free if there is zero velocity towards the obstacle on the surface, i.e.,
\begin{equation}
	\left| \vect n(\vecs \xi)^T \, \{ \vecs{\dot \xi} \}_{t} \right| = 0 
	\quad \forall \, \Gamma(\vecs \xi) = 1
	\label{eq:bezis_theorem}
\end{equation}

We want to find the time when the agent stops moving towards the obstacle, enabling us to evaluate the distance traveled as a function of the velocity after disturbance $\vect v^I$. 
Let us assume without loss of generality that the disturbance occurs at time $t=0$. Hence, the velocity at time $T$ can be computed as:
\begin{equation}
\begin{split}
	\{ \vecs{\dot \xi} \}_{T} 
	& = \int_0^T \vecs{\ddot \xi} \, dt = \int_0^T \matd{M}^{-1}(\cdot) \matd{D}(\cdot) \left( \vecs{\dot \xi} - \vect f(\vecs \xi ) \right) \, dt \\
	& = \matd{M}^{-1}(\cdot) \int \left( (1 - w) \matd{D}^f(\cdot) + w \matd{D}^{o}(\cdot) \right) \left( \vecs{\dot \xi} - \vect f (\vecs \xi ) \right) dt \\
	\end{split}
\label{eq:velocity_evolution}
\end{equation}

Furthermore, since the vector field, $\vect f(\vecs \xi)$ is constant and the obstacle's surface does not have any curvature, it follows from \eqref{eq:damping_summation} that the damping matrices  $\matd{D}^o$ and $\matd{D}^f$ are constant.
Moreover, by design of the damping matrices, from \eqref{eq:first_obstacle_basis} it follow that $\matd{D}^o(\vecs \xi, \dot{\vecs \xi}) \vect n(\vecs \xi) = s^o \vect n(\vecs \xi)$, and from  \eqref{eq:velocity_unit_vector} that $\matd{D}^f(\vecs \xi, \dot{\vecs \xi}) \vect n(\vecs \xi) = s^f \vect n(\vecs \xi)$. 

From \eqref{eq:bezis_theorem} follows that it is sufficient to observe the normal component of the vectors only. Thus, in the rest of this paragraph, the components along the normal are denoted by scalar values, e.g. $\dot \xi = \dotprod{\dot{\vecs \xi}}{\vecs n(\vecs \xi)} \vecs n(\vecs \xi)$.
Hence, we get:
\begin{align}
	0 & = \left| \vect n^T \, \{ \vecs{\dot \xi} \}_{t} \right| 
	  = \frac 1 m \int_0^T ( (1-w) \vect n^T \matd{D}^f \vect n + w s^{o} ) \, {\dot \xi} \, dt \nonumber \\
	   & < \frac 1 m \int_0^T  s^{f}  \, {\dot \xi} \, dt 
	   = \frac{s^f}{m} (\xi - \{ \xi \}_0 ) + v^I \, dt 
\end{align}
where $m = \max{\left(\text{eig}(\matd{M})\right)}$, with the maximum displacement as:
\begin{equation}
	\| \{\xi \}_0 - {\xi} \| \leq \| v^I \| {m} / {s^{\mathrm{f}}} 
\end{equation}
\end{proof}
\else
\editcolor{
\begin{proof}
See extended artice \parencite{passive2024huber_extended}. \hfill 
\end{proof}
}
\fi
Lemma~\ref{lemma:damping_collision_avoidance} assumes constant velocity field $\vect f(\vecs \xi)$ and flat obstacle surface. This is an appropriate assumption for large velocities after disturbances towards the obstacle, i.e., $\| \vect v^I \| \gg \| \vect f(\vecs \xi) \|$ and starting close to the surface. Since the distance traveled has to be flat to avoid collision, the vectorfield is likely to show small changes, and the surface has little deviation. 

\begin{figure}[htb]
\centering
 \iflong
  \centerline{\includegraphics[width=0.99\columnwidth]{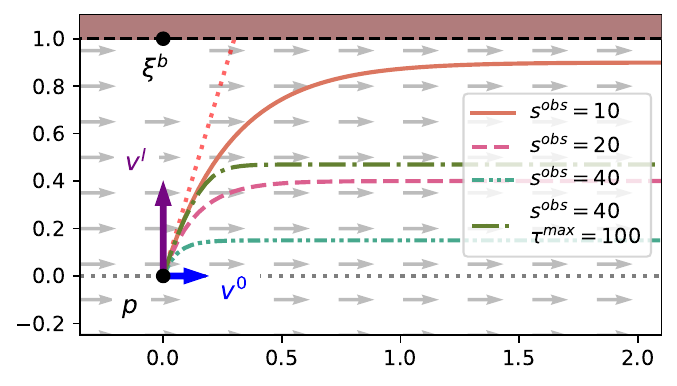}}
  \else
  \centerline{\includegraphics[width=0.99\columnwidth]{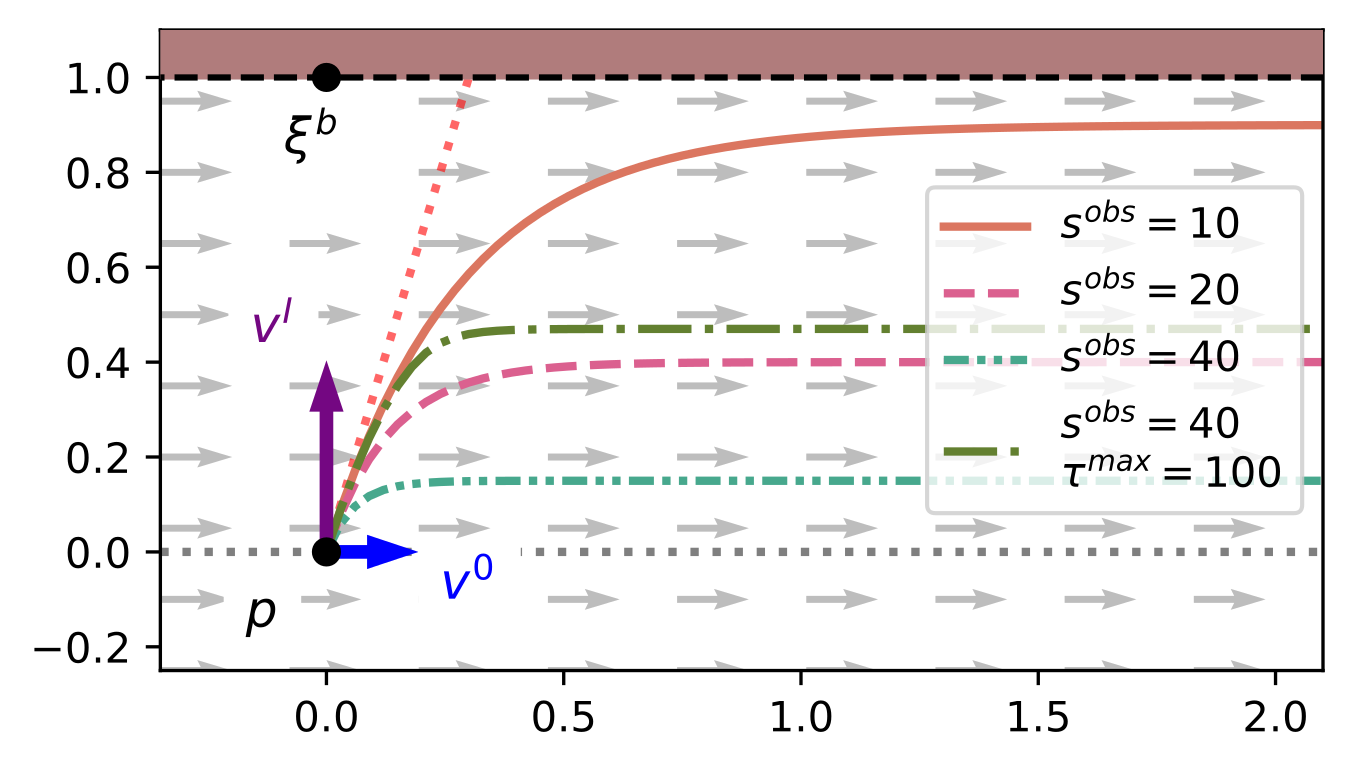}}
  \fi
  \caption{A disturbance occurs of a point-agent at position $\vect p^0$ with velocity after the impact of $\{ \dot{\vecs \xi} \}_0 = \vect v^0 + \vect v^I$. A high damping in the direction of the obstacle in the presence of a constant velocity field (gray) ensures collision avoidance. Whereas different damping values $s^{\mathrm{o}}$ and optionally a maximum repulsion force $\vecs \tau^{\mathrm{max}}$ lead to different trajectories.}
  \label{fig:disturbance_with_parallel_velocity}
\end{figure}
    
Nevertheless, there is no guarantee against drifting into obstacles in the presence of highly curved surfaces and velocity fields. \iflong This is further discussed during the experiments in Section~\ref{sec:position_noise}, where the increased damping towards the obstacle significantly reduces collision in such scenarios. \fi However, designing a repulsive field as proposed in \parencite{huber2023avoidance} can ensure collision avoidance in such scenarios.

\iflong
All robotic systems have a maximum force that they can exert on the environment based on the motors, geometry, and state, $\tau_c^{\mathrm{max}} \in \mathbb{R}_{>0}$. Such a limiting force decreases the impact velocity a controller can handle while ensuring collision avoidance, as shown in Fig.~\ref{fig:disturbance_with_parallel_velocity}. Nevertheless, a maximum control force can be interpreted as adapting damping; hence, the passivity from Theorem~\ref{theorem:passivity} holds.
\fi

\section{Evaluation}  \label{sec:evaluation} 
The proposed obstacle-aware controller\footnote{Source code: \url{https://github.com/hubernikus/obstacle_aware_damping}} is compared to a baseline, the velocity preserving, passive controller \parencite{kronander2015passive}.

\iflong
\subsection{Qualitative Repulse Rejection} \label{sec:qual_comp}
A qualitative analysis of the proposed controller's behavior in three scenarios, as depicted in Figure~\ref{fig:obstacle_aware_damping_comparison}. In each scenario, the agent approaches multiple obstacles from distinct starting positions, and a disturbance (indicated by an arrow) is applied. The simulation time step is $\Delta t = 0.01 s$ seconds, and the agent's mass matrix is $\matd{M} = \matr{I} kg$. The controller is implemented using the following damping values:
$s^{\mathrm{obs}}=~$\qty{200}{s^{-1}},
$s^{\mathrm{f}}=$~\qty{100}{s^{-1}}, and
$s^{\mathrm{c}}=$~\qty{20}{s^{-1}}.

\begin{figure}[htbp]
  \centering
  \ifthesis
  \centerline{\includegraphics[width=0.8\columnwidth]{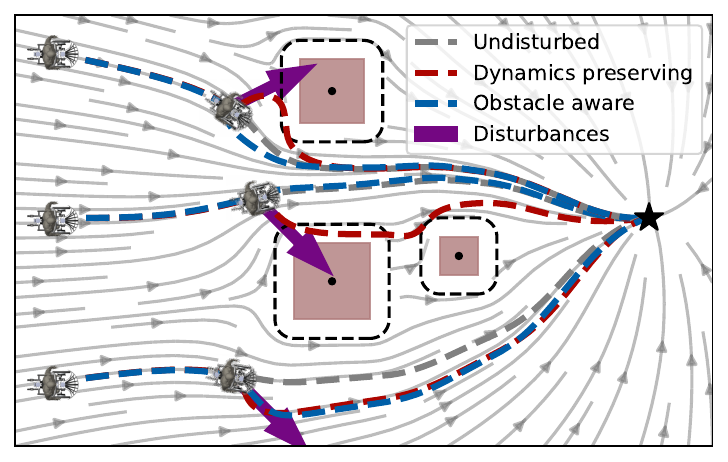}}
  \else
  \centerline{\includegraphics[width=0.95\columnwidth]{figures/multi_obstacle_with_damping.pdf}}
  \fi
  \caption{
  The desired velocity $\vect f(\vecs \xi)$ is represented in gray and serves as the input for the force controller. The mobile robot, initially positioned at three different locations, navigates safely towards the attractor (black star) even when confronted with external disturbances (purple arrows) while employing the obstacle-aware controller (blue trajectories). In contrast, the baseline controller (red) results in collisions when disturbances occur close to the robot.
  }
  \label{fig:obstacle_aware_damping_comparison}
\end{figure}

In the top trajectory, the robot encounters a stand-alone obstacle and experiences a disturbance that pushes it toward the obstacle. With the obstacle-aware controller, the robot avoids collision and continues moving towards the attractor. In contrast, the baseline controller, which prioritizes velocity conservation, fails to respond effectively and is pushed into the obstacle, resulting in a colliding trajectory. 
	During the middle trajectory, a disturbance occurs when the agent is positioned between two obstacles. The obstacle-aware controller utilizes the normal vector $\vect n(\vecs \xi)$, as described in Section~\ref{sec:obstacle_repulsion}. However, due to its construction, the magnitude of $\vect n(\vecs \xi)$ diminishes in narrow passages (as seen in \eqref{eq:averaged_normal}), leading to increased damping in all directions, as described in \eqref{eq:leaving_compliance}. As a result, the agent successfully avoids the disturbance using the obstacle-aware controller while the baseline controller follows a colliding trajectory.
In the bottom trajectory, the repulsive force points away from the obstacle. The obstacle-aware and the velocity-preserving controllers produce nearly identical trajectories due to equal compliance when moving away from an obstacle as defined in \eqref{eq:leaving_compliance}.
We observe the selective damping of disturbances towards the obstacle. The top and middle disturbances are highly damped, while the bottom disturbances are not. This feature imparts a natural behavior of moving away from obstacles.
\fi


\subsection{Noise Analysis}
\editcolor{
In practical robotic applications, controllers are inevitably exposed to noise and unexpected disturbances arising from sensor inaccuracies, environmental factors, or other external perturbations. A high-quality controller can effectively reject such disturbances while simultaneously achieving the control objectives, such as collision avoidance and trajectory tracking.
This section investigates the impact of noise disturbances on a simulated agent with an identity mass matrix $\matd M = \matr I kg$. The discrete time step is set to $\Delta t = $~\qty{0.2}{s}. Additionally, the damping values are configured as follows: 
$s^{\mathrm{o}}=$~\qty{50}{s^{-1}},
$s^{\mathrm{f}}=$~\qty{40}{s^{-1]}}, and
$s^{\mathrm{c}}=$~\qty{5}{s^{-1}}.
The robot's objective is to follow a linear velocity field of the form $\vect f(\vecs \xi) = (\vecs \xi^a - \vecs \xi)$ with a velocity cap of \qty{1}{m/s}.
To evaluate the controller's performance, a comparative analysis is conducted by assessing the minimal distance to the surface along the trajectory, denoted as $ \min_t \| \vecs \xi_t - \vecs \xi^b \| $, with the closest boundary point $\vecs \xi_b$ described in \eqref{eq:distance_function}. 
}

\begin{figure}[htbp]
    \centering
    \begin{subfigure}{\columnwidth}
	\centering
	\ifthesis
    \includegraphics[width=0.7\textwidth]{figures/trajectory_velocity_noise}
	\else
	  \iflong
      \includegraphics[width=\textwidth]{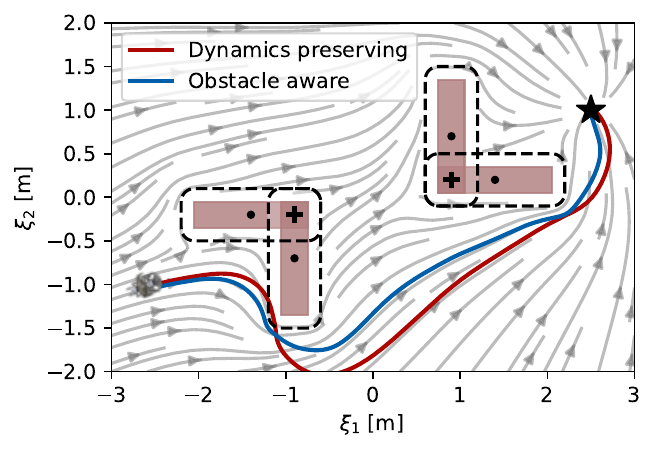}
	  \else
      \includegraphics[width=\textwidth]{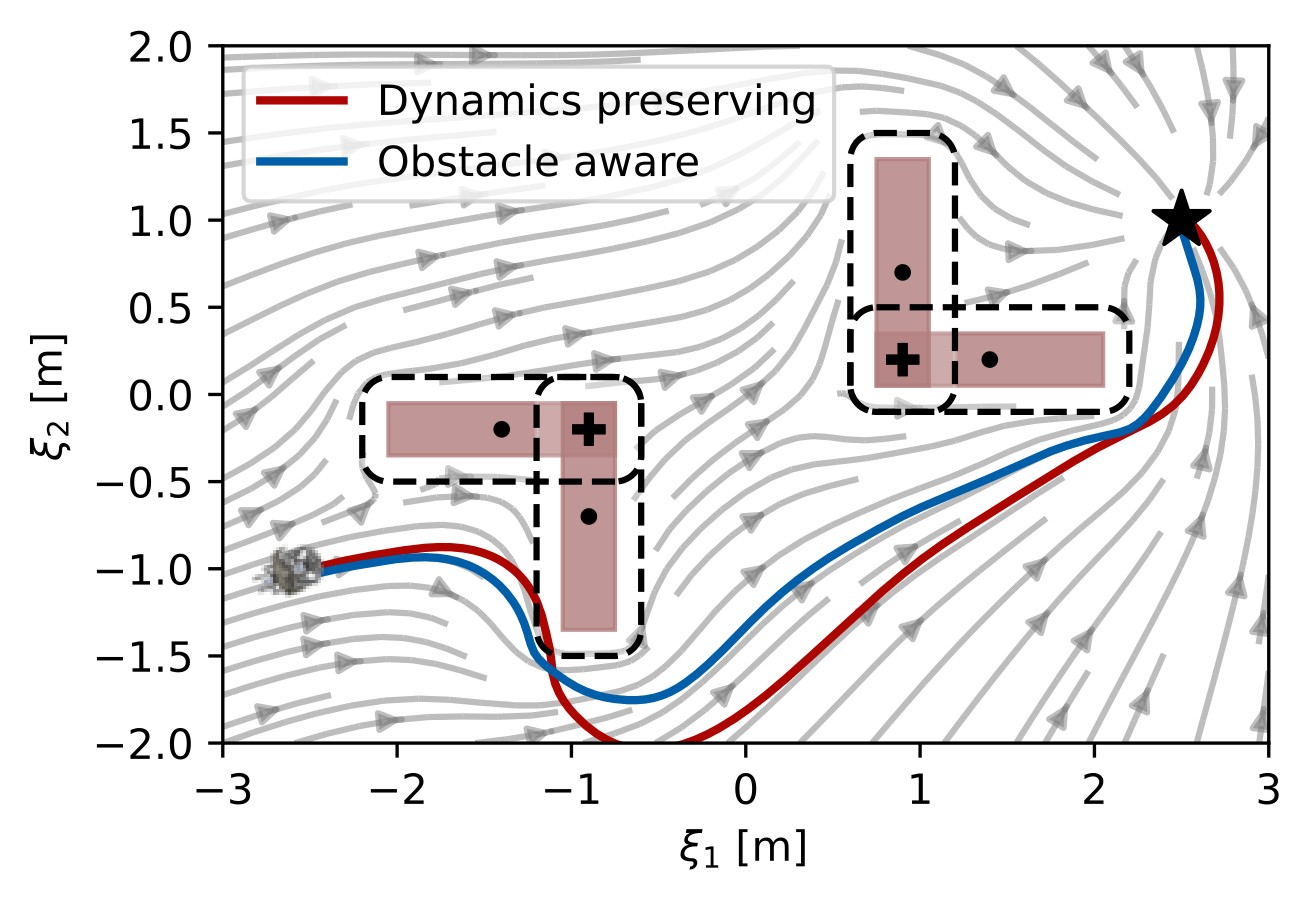}
	  \fi
	\fi
    \caption{Trajectories with a standard deviation of the velocity-noise of 1.0 m/s.}
    \label{fig:trajectory_velocity_noise}
    \end{subfigure}
    \begin{subfigure}{\columnwidth}
	\centering
	\ifthesis
    \includegraphics[width=0.7\textwidth]{figures/comparison_velocity_noise}
	\else
	  \iflong
	  \includegraphics[width=\textwidth]{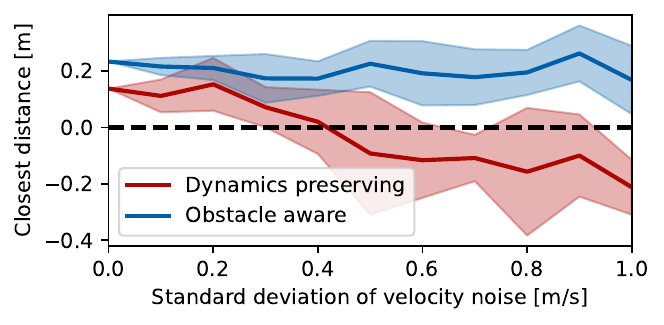}
	  \else
	  \includegraphics[width=\textwidth]{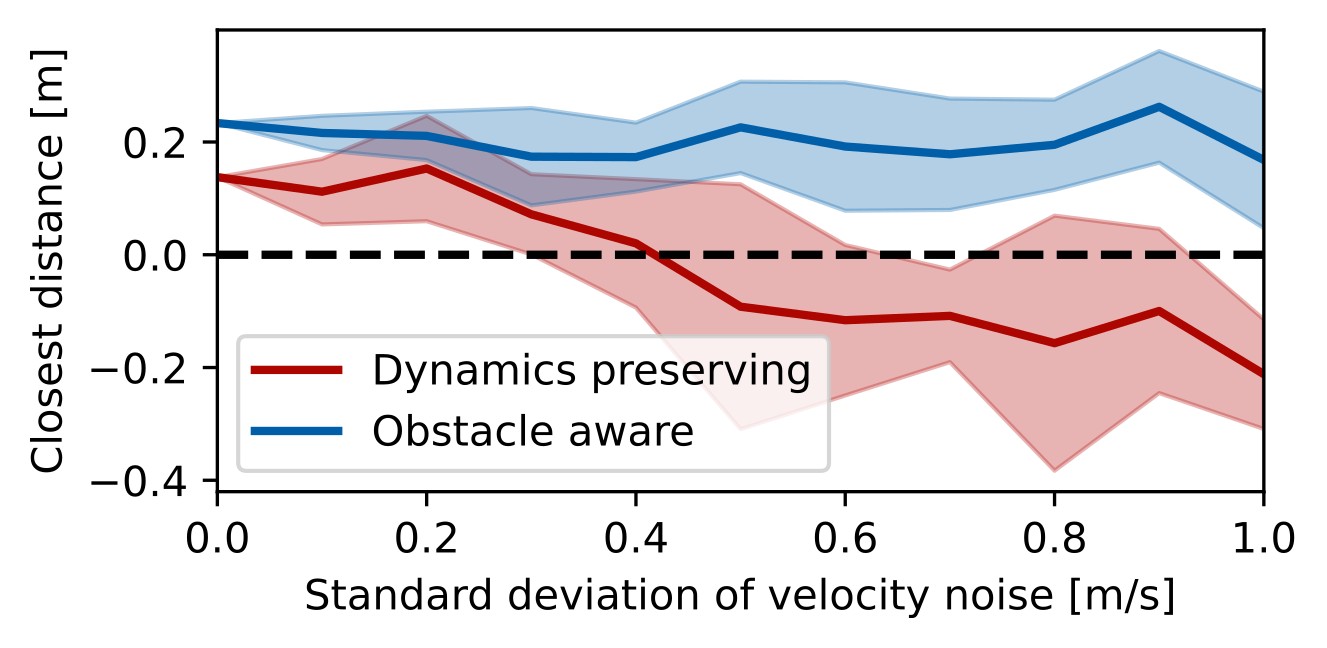}
	  \fi
	\fi
    \caption{The mean and variance (shaded) of the closest distances over 10 epochs.}
    \label{fig:comparison_velocity_noise}
    \end{subfigure}
	\caption{The agent navigates between two elongated obstacles (a) towards the attractor (black star). The agent's velocity is subjected to white noise, with a mean of zero, and noise variances between \qty{0.0}{m/s} and \qty{1.0}{m/s}. The robot initiates its trajectory from the starting position $\vecs \xi_0 = [-2.5, 1.0]^T$, with an initial velocity of zero. It aims to reach the attractor located at $\vecs \xi^a = [2.5, \, -1.0]^T$.}
\label{fig:velocity_noise}
\end{figure}


\subsubsection{Velocity Noise Resistance}
\editcolor{
We first added a normally distributed noise with a zero mean to the agent's velocity $\dot{\vecs{ \xi}}$ at each time step before computing the control force. The agent has to navigate between two obstacles from the starting position to the attractor with different noise levels, as visualized in Figure~\ref{fig:velocity_noise}). }

\editcolor{
The obstacle-aware controller effectively rejects the noise impacting the velocity, even as the noise variance increases. However, the closest distance during the trajectory diminishes with higher noise variance. On the contrary, the velocity-following controller's mean distance falls below zero already at a velocity noise variance of $\qty{0.5}{m/s}$, indicating that many trajectories collide with at least one obstacle.
Furthermore, the obstacle-aware controller maintains a higher minimal distance along the trajectory without noise. This effect results from the higher damping applied towards the obstacle, enabling more precise tracking of the curvature guiding the velocity around the obstacle.
}

\subsubsection{Position Noise Resistance} \label{sec:position_noise}
\editcolor{
In the second experiment, normally distributed noise with a mean of zero is added to the position $\vecs{ \xi}$ (Fig.~\ref{fig:position_noise}). The agent has to navigate around two star-shaped obstacles to reach the attractor with various noise levels being applied.}

\editcolor{The obstacle-aware controller maintains a greater distance to the surface even when no noise is present, while the velocity-preserving controller already exhibits collisions. As a result, the obstacle-aware creates trajectories with the mean distance above zero for standard deviations of the position noise smaller than \qty{0.023}{m}. Conversely, the velocity-preserving controller's mean distance to the surface is below zero for all experiment runs, indicating collisions.
This difference is attributed to the buffer mechanism inherent in the obstacle-aware controller. Despite a similar decrease in distance for both controllers, the obstacle-aware controller effectively prevents collisions due to its higher damping of velocities toward the obstacles.
Moreover, the velocity-preserving controller exhibits a higher distance variance, likely stemming from its lower damping, causing it to adapt more slowly to new velocities after being displaced by the noise. Consequently, this behavior leads to more random variations in velocity and trajectory.
}
\begin{figure}[htbp]
    \centering
    \begin{subfigure}{\columnwidth}
	\centering
	\ifthesis
    \centerline{\includegraphics[width=0.7\textwidth]{figures/trajectory_position_noise}}
	\else
	  \iflong
	  \centerline{\includegraphics[width=\textwidth]{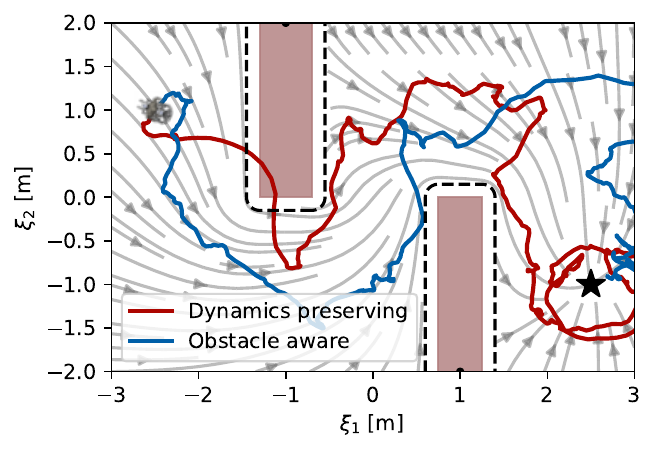}}
	  \else
	  \centerline{\includegraphics[width=\textwidth]{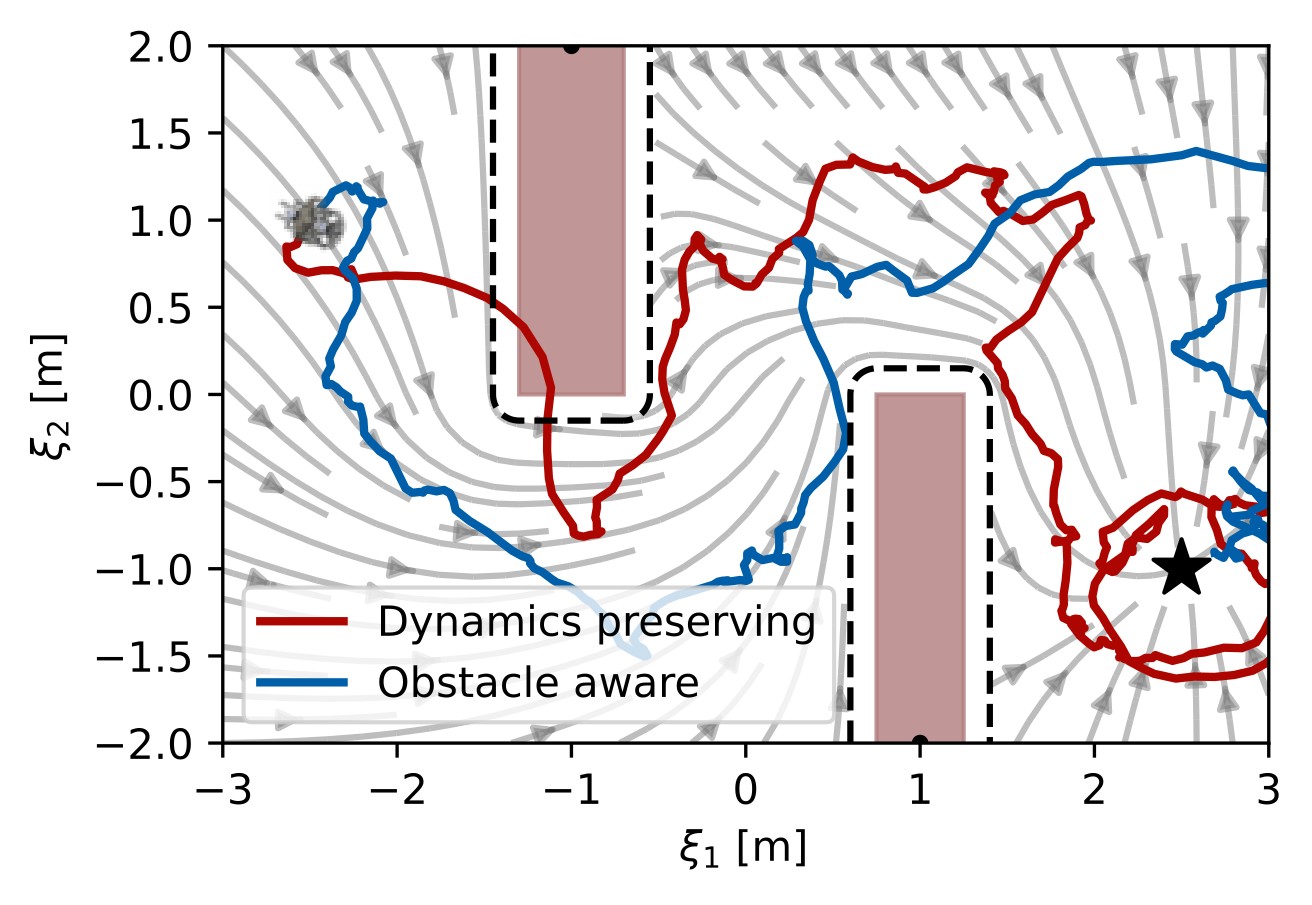}}
	  \fi
	\fi
    \caption{Trajectories with a standard deviation of the position-noise of 0.0 m.}
    \label{fig:trajectory_position_noise}
    \end{subfigure}
    \begin{subfigure}{\columnwidth}
	\centering
	\ifthesis
    \includegraphics[width=0.7\textwidth]{figures/comparison_position_noise}
	\else
	  \iflong
      \includegraphics[width=\textwidth]{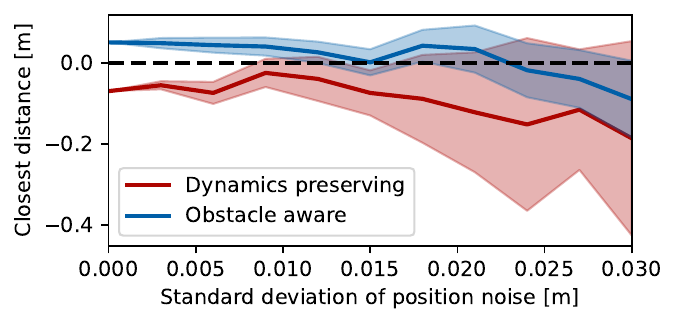}
	  \else
      \includegraphics[width=\textwidth]{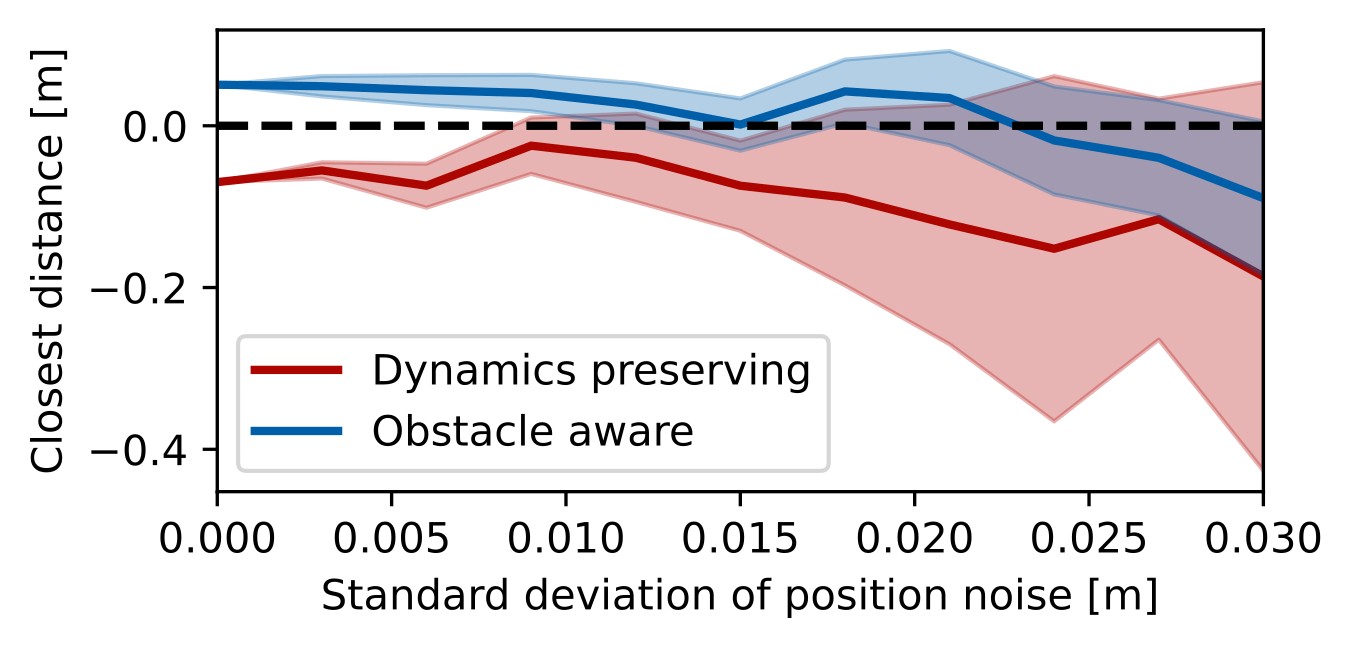}
	  \fi
	\fi
    \caption{Closest distances concerning different noise levels over ten epochs.}
    \label{fig:comparison_position_noise}
    \end{subfigure}
	\caption{
 The agent is navigating towards the attractor (black star) between two concave obstacles (a) while being subjected to white noise in its position. The position noise has a mean of zero and various noise variances ranging between \qty{0.0}{m} and \qty{0.03}{m}. The robot starts at position $\vecs \xi_0 = [-2.5, -1.0]^T$, and the attractor is set to $\vecs \xi^a = [2.5, 1.0]^T$.
 }
\label{fig:position_noise}
\end{figure}

\subsection{Obstacle-Aware Passivity Using a Robot Arm}
The obstacle-aware passivity controller was implemented to guide a 7-degree-of-freedom robot arm (Panda from Franka Emika) while moving around a cubic obstacle. 

\ifthesis
\begin{figure}[htbp]
    \centering
   \begin{subfigure}{\columnwidth}
    \includegraphics[width=0.4\textwidth]{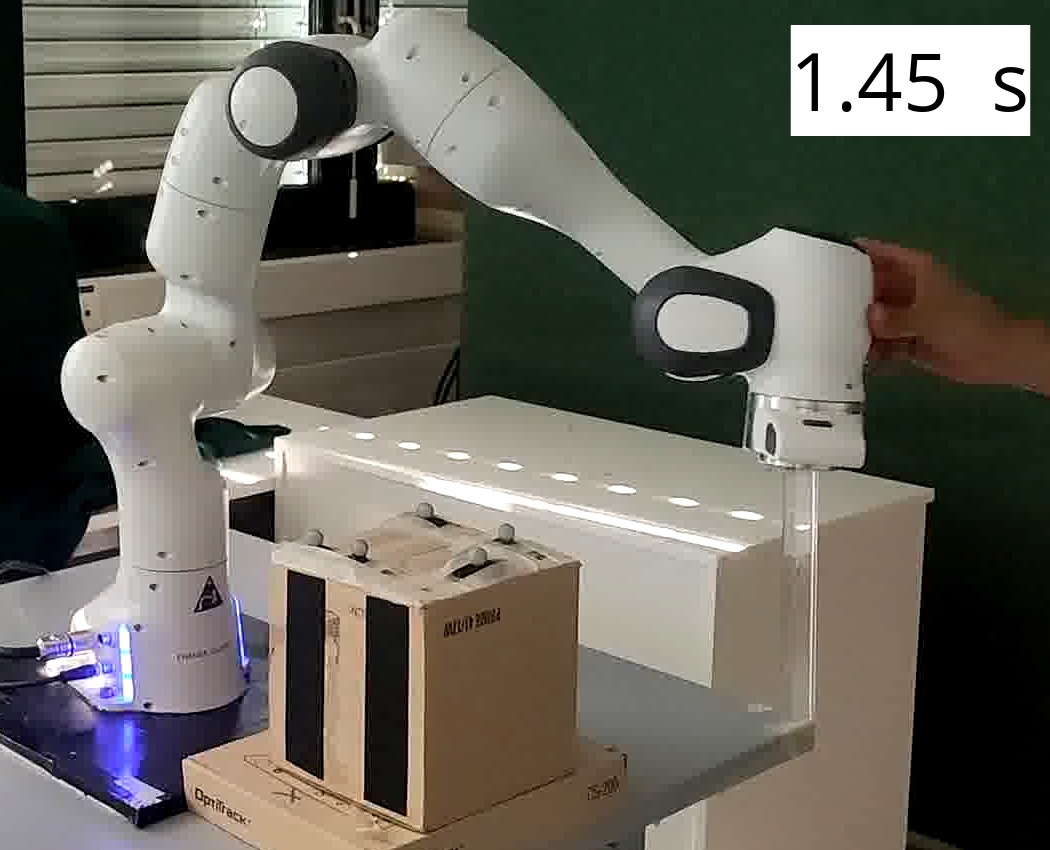}\hfill%
    \includegraphics[width=0.4\textwidth]{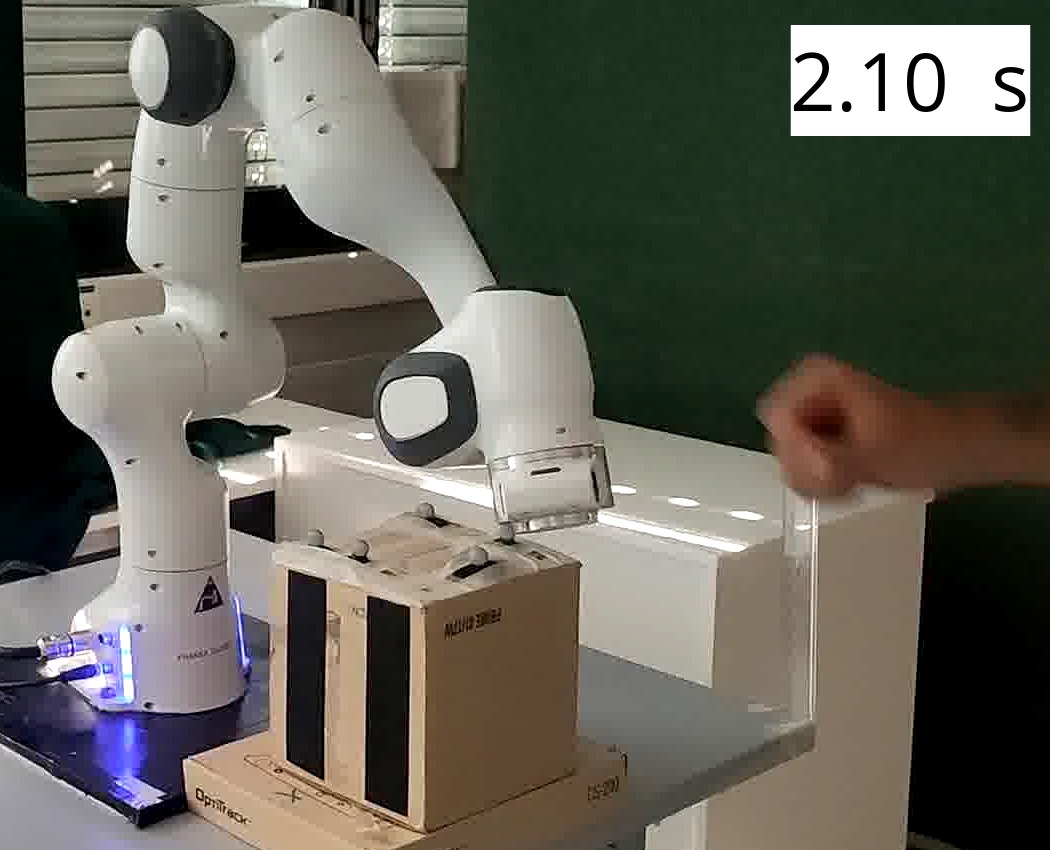}
      \caption{Obstacle-aware controller rejects repulsion and avoids collision}
      \label{fig:franka_sequence_obstacle_aware}
    \end{subfigure}
	\begin{subfigure}{\columnwidth}
    \includegraphics[width=0.4\textwidth]{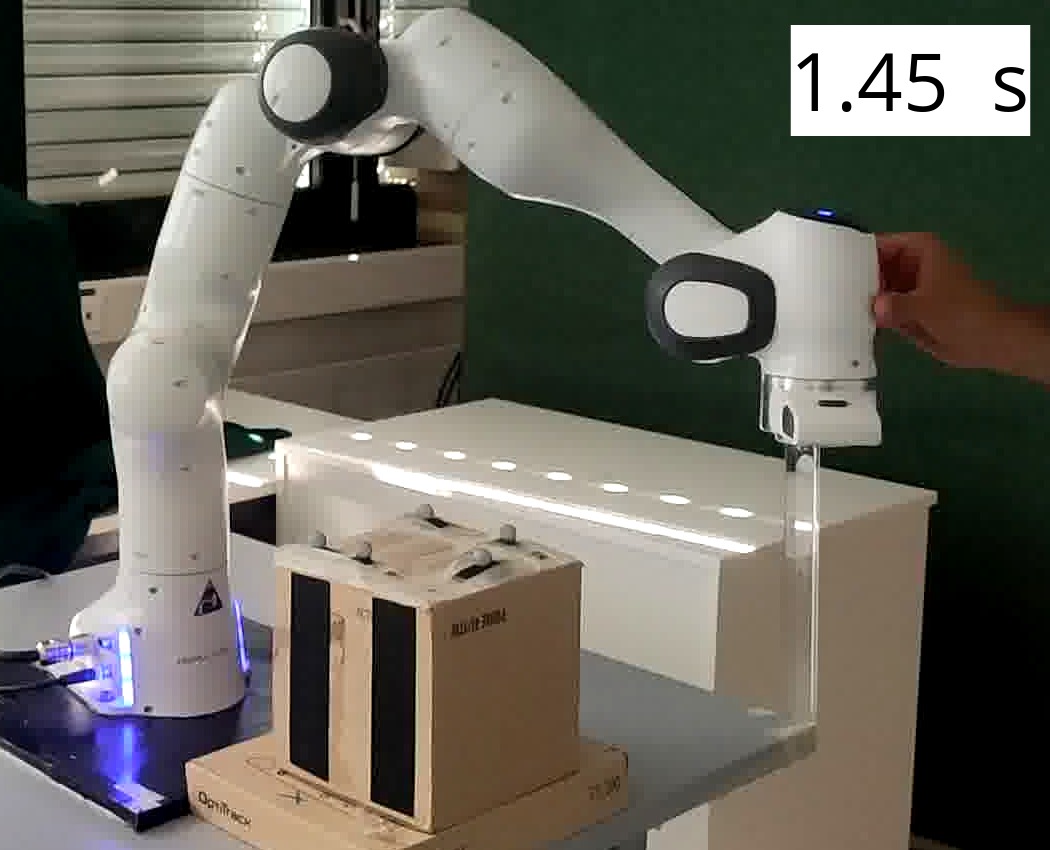}\hfill%
    \includegraphics[width=0.4\textwidth]{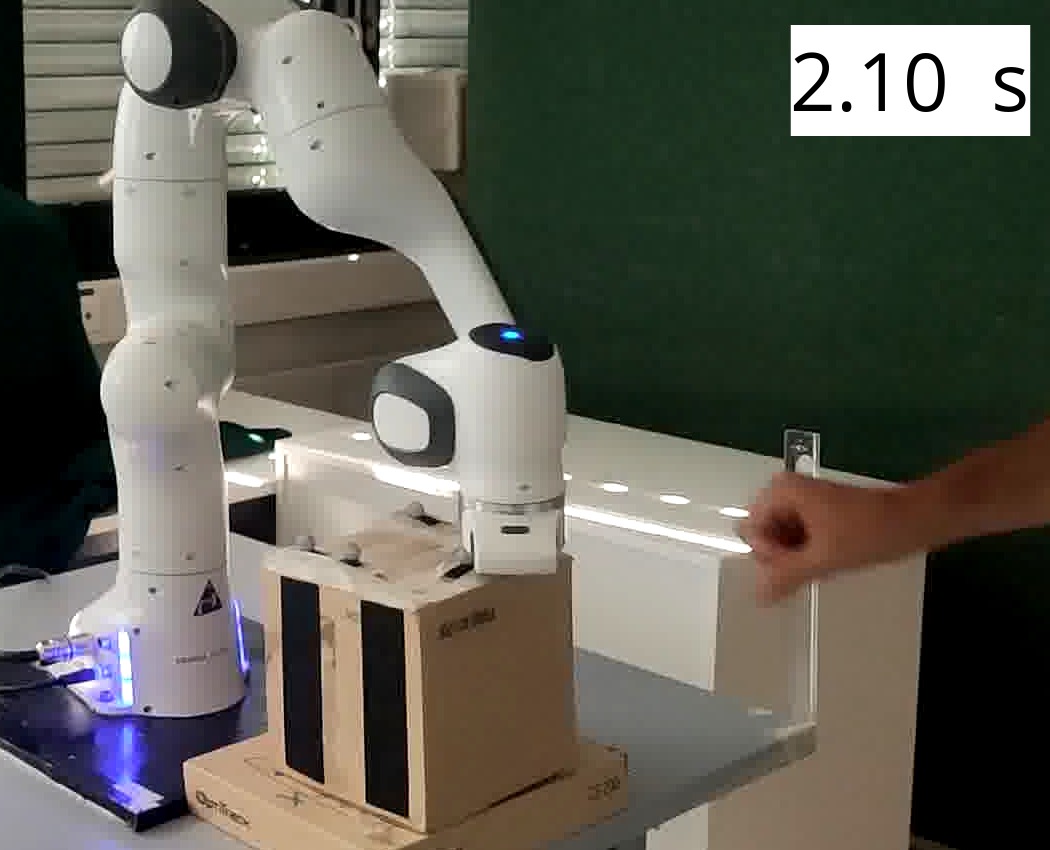}\hfill%
      \caption{Velocity preserving controller leads to collision with obstacle}
      \label{fig:franka_sequence_obstacle_aware}
    \end{subfigure}
	\caption{The robot arm, guided by the obstacle-aware passive controller, effectively avoids the disturbance towards the obstacle while maintaining a margin of \qty{0.16}{m} around the obstacle.}
	\label{fig:instances_evaluation_on_robot_arm}
\end{figure}

\begin{figure}[htbp]
    \begin{subfigure}{\columnwidth}
      \centerline{\includegraphics[width=\textwidth]{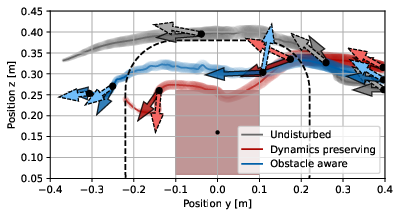}}
      \caption{The two control methods compared with the undisturbed trajectory. The wider line indicates a higher x-value. The darker arrow is the actual, and the desired velocity is the brighter arrow.}
      \label{fig:robot_arm_trajectory_xyz}
    \end{subfigure}
    \begin{subfigure}{\columnwidth}
		\includegraphics[width=\textwidth]{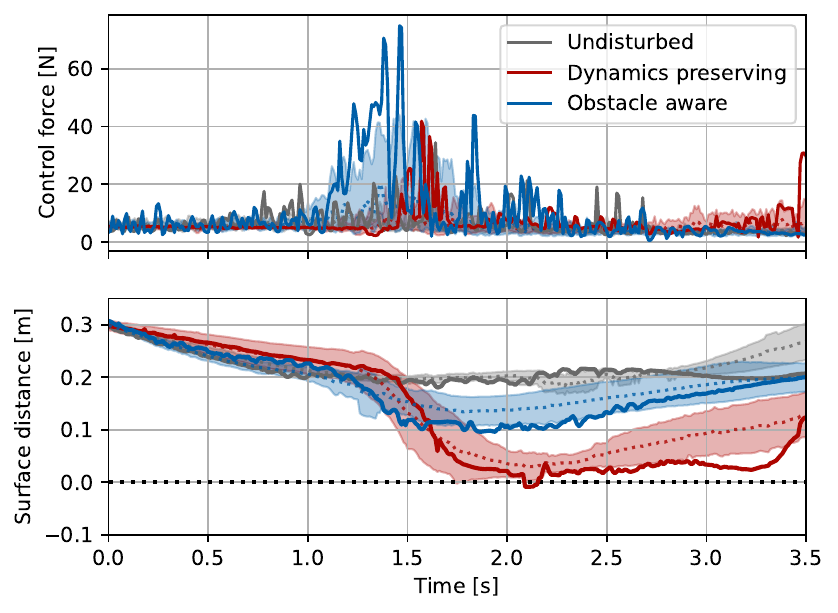}
      \caption{The specific trajectory is represented by a full line, while the average (dashed line) and variance (shaded area) are evaluated over ten epochs. The mean and variance of the control force are evaluated in the logarithmic space.}
      \label{fig:trajectory_comparison_force_and_distance}
    \end{subfigure}
	\caption{
 The experiment was repeated ten times with a similar disturbance applied to the robot arm in each run.
 }  
    \label{fig:evaluation_on_robot_arm}
\end{figure}
\
\else
\begin{figure}
    \centering
   \begin{subfigure}{\columnwidth}
    \includegraphics[width=0.49\textwidth]{figures/franka_sequence/franka_obstacle_aware016}\hfill%
    \includegraphics[width=0.49\textwidth]{figures/franka_sequence/franka_obstacle_aware020}
      \caption{Obstacle-aware controller rejects repulsion and avoids collision}
      \label{fig:franka_sequence_obstacle_aware}
    \end{subfigure}
	\begin{subfigure}{\columnwidth}
    \includegraphics[width=0.49\textwidth]{figures/franka_sequence/franka_velocity_conserving021}\hfill%
    \includegraphics[width=0.49\textwidth]{figures/franka_sequence/franka_velocity_conserving025}\hfill%
      \caption{Velocity preserving controller leads to collision with obstacle}
      \label{fig:franka_sequence_obstacle_aware}
    \end{subfigure}
    \begin{subfigure}{\columnwidth}
      \centerline{\includegraphics[width=\textwidth]{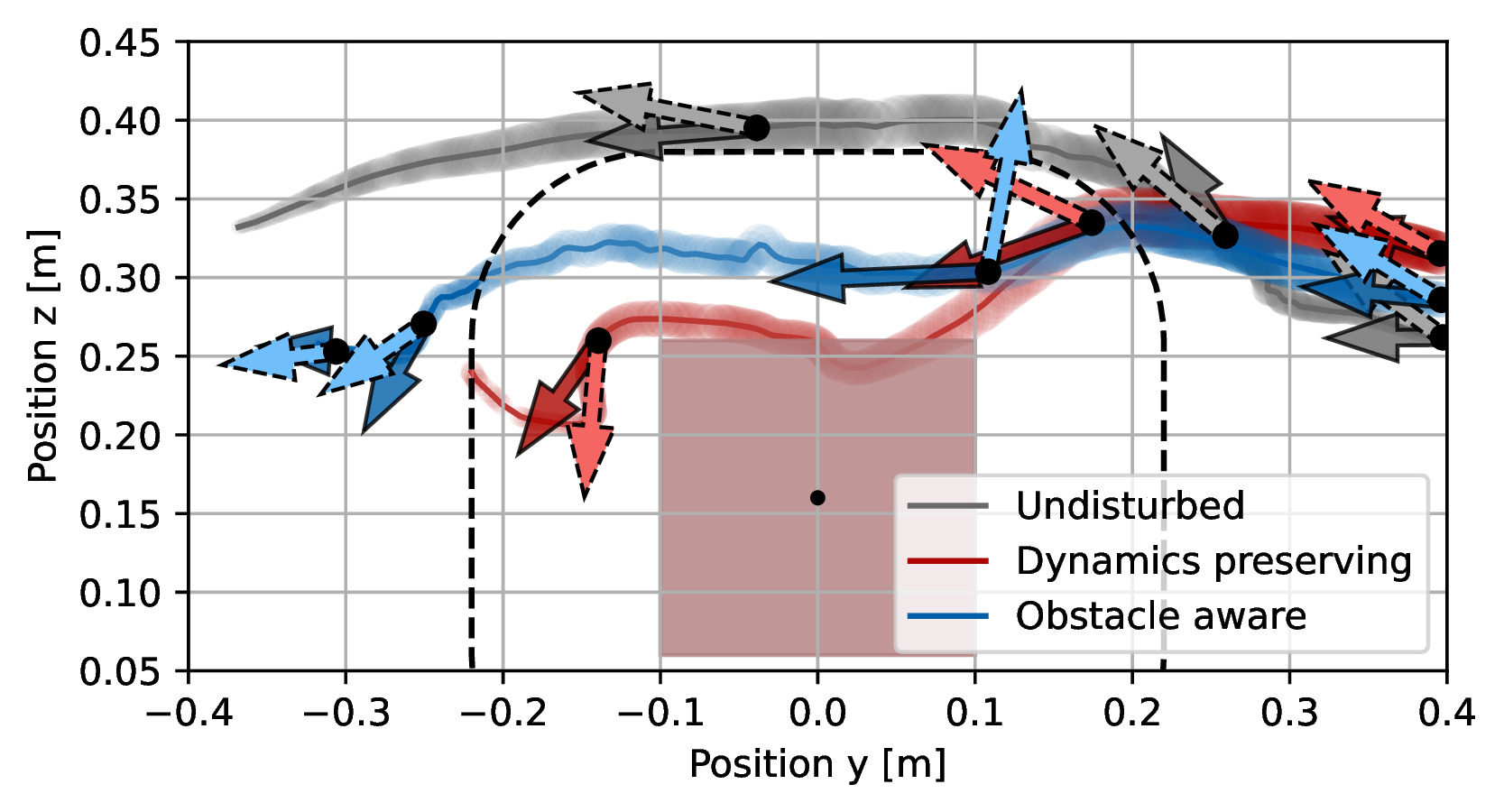}}
      \caption{The two control methods compared with the undisturbed trajectory. The wider line indicates a higher x-value. The darker arrow is the actual, and the desired velocity is the brighter arrow.}
      \label{fig:robot_arm_trajectory_xyz}
    \end{subfigure}
    \begin{subfigure}{\columnwidth}
		\includegraphics[width=\textwidth]{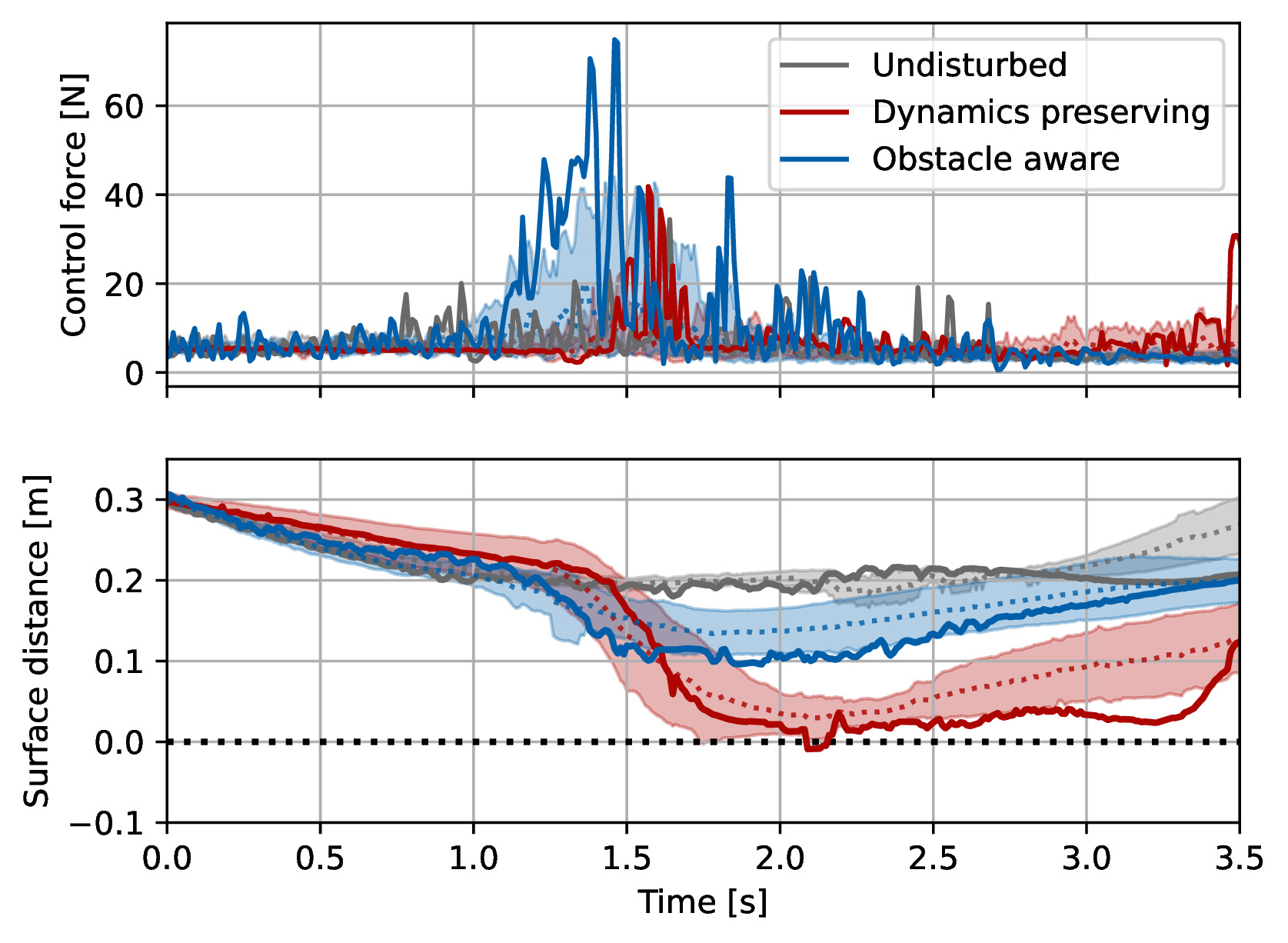}
      \caption{The specific trajectory is represented by a full line, while the average (dashed line) and variance (shaded area) are evaluated over ten epochs. The mean and variance of the control force are evaluated in the logarithmic space.}
      \label{fig:trajectory_comparison_force_and_distance}
    \end{subfigure}
	\caption{
The robot arm, guided by the obstacle-aware passive controller, effectively avoids the disturbance towards the obstacle while maintaining a margin of \qty{0.16}{m} around the obstacle. The experiment was repeated ten times with a similar disturbance applied to the robot arm in each run.
 }  
    \label{fig:evaluation_on_robot_arm}
\end{figure}
\fi

The joint torque is computed using inverse kinematics combined with the proposed passive controller for the position, but a proportional controller for the orientation:
\begin{equation}
	\vecs \tau_q = \matr J^{\dag}(\vect q) 
	\begin{bmatrix} \matd D(\vecs \xi) \left( \vect f(\vecs \xi) - \dot{\vecs \xi} \right) \\  p^\alpha (\vecs \alpha - \vecs \alpha^a) \end{bmatrix}
\end{equation}
where $\matr J^{\dag}$ represents the Moore-Penrose pseudo inverse of the Jacobian matrix, and $\vecs \alpha$ and $\vecs \alpha^a$ denote the end effector's orientation and the desired orientation, respectively. The angular damping parameter is chosen as $p^\alpha = 5.5$.
The desired orientation $\vecs \alpha^a$ is pointing downward with a quaternion value of $(w, x, y, z) = (0, 1, 0, 0)$. For the angle subtraction, we use quaternion representation to avoid singularities, but an angle-axis representation of the orientation is used to evaluate the torque from the angular offset.

The angular damping is chosen as $p^\alpha = 5.5$.
The damping values are set as
$s^{\mathrm{o}}=$~\qty{160}{s^{-1}},
$s^{\mathrm{f}}=$~\qty{64}{s^{-1]}}, and
$s^{\mathrm{c}}=$~\qty{16}{s^{-1}}.

The robot start position is approximately at $\vecs \xi_0 = [0.3\mathrm{m}, 0.4\mathrm{m}, 0.3\mathrm{m},]^T$ and the attractor is at position $\vecs \xi^a = [0.26 \mathrm{m}, -0.53\mathrm{m}, 0.33\mathrm{m}]^T$. The controller operates at a frequency of \qty{500}{Hz}.
The robot encounters a single squared box with axes length \qty{0.16}{m} and a margin of \qty{0.12}{m}, placed in front of the robot base (Fig.~\ref{fig:evaluation_on_robot_arm}). The precise location of the box is tracked in real-time using a marker-based vision system (Optitrack). When passing the box, the robot is pushed with $\vecs t^e$ towards the box. The experiment is repeated ten times for both controllers and compared to the undisturbed motion.

\iflong
The experimental results in \ifthesis Figure~\ref{fig:instances_evaluation_on_robot_arm} and \fi Figure~\ref{fig:evaluation_on_robot_arm} demonstrate that the obstacle-aware controller allows the robot to maintain an average distance of \qty{0.15}{m} away from the obstacle surface. In contrast, the velocity-preserving controller results in a mean distance below \qty{0.05}{m}, and numerous trajectories lead to collisions with the box. On average, the robot passes further away from the obstacle using the obstacle-aware controller and exhibits higher forces, Table~\ref{tab:evaluation_on_robot_arm}.

\begin{table}[htb]
    \centering
    \begin{tabular}{|l|c|c||c|} \hline
        & Obstacle & Velocity & Undisturbed \\ \hline
        Closest Distance [mm] &  {105 $\pm$ 21} & 18 $\pm$ 21 & 167 $\pm$ 9 \\ \hline
        Maximum Force [N] & {4.12 $\pm$ 0.25} & 3.58 $\pm$ 0.24 & 2.86 $\pm$ 0.37  \\ \hline 
    \end{tabular}
    \caption{The mean and standard deviation of the closest distance and maximum force are evaluated over ten epochs for the proposed obstacle aware controller, and the baseline velocity following controller \parencite{kronander2015passive}.
    The undisturbed trajectory is provided as reference values.
    During the experiments, the obstacle-aware damping avoids collision with an increased minimal distance from the obstacle and, hence, a safer trajectory. This is achieved through momentarily higher forces.}
    \label{tab:evaluation_on_robot_arm}
\end{table}
\fi

This outcome is attributed to the obstacle-aware controller's stronger control force, with a high peak occurring around \qty{1.45}{s} when the disturbance is encountered. In contrast, the velocity-preserving controller only acts when the robot almost collides, leading to a delayed response. Additionally, the obstacle-aware controller exhibits higher forces, contributing to improved tracking of the avoidance trajectory\iflong, as observed in Section~\ref{sec:position_noise}\fi. These findings affirm the superior collision avoidance capabilities and tracking performance of the obstacle-aware passivity controller in real-world robot arm scenarios.

\section{Discussion}
We introduced a novel passive obstacle-aware controller that takes as an input the desired, collision-free velocity and outputs a control torque. 
The stability proof enables the controller to be used with any bounded input velocity field, and this result extends to a general class of damping controllers.
Furthermore, the controller is shown to reject disturbances, and the parameter-tuning for the discrete-time system has been analyzed. 
The controller was experimentally evaluated and compared to a baseline passive controller. 
\iflong The proposed approach shows increased resistance to noise, both in position and velocity and improved tracking. \fi
Applied to a real robot arm, the disturbance force was successfully rejected, ensuring collision avoidance while following a reference motion.

\iflong
\subsection{Discrete Control Convergence}
The proposed controller is only stable but does not ensure convergence. However, convergence could be achieved by modifying the controller by introducing an additional proportional term, as is common in impedance controllers:
\begin{equation}
	\vecs{\tau_c} = \vect g (\vecs\xi) 
	+ \matd{D}(\vecs\xi) \left(\vecs f(\vecs\xi) - \vecs{\dot\xi} \right) 
	+ \matd{K}(\vecs \xi) \left(\vecs \xi - \vecs \xi^a \right)
\label{eq:control_command_proportional}
\end{equation}
However, this design interferes with the basis passive controller, and the convergence towards the attractor is facilitated by (stable) velocity $\vect f(\vecs \xi)$. 
Moreover, introducing a variable damping matrix will result in a potentially unstable system, as the passivity proof no longer holds.
Future work should further analyze how to include dynamics properties in the design of the damping parameters for improved system performance.
\fi

\subsection{Applicability and Theoretical Analysis}
The theoretical analysis from Theorem~\ref{theorem:passivity} indicates BIBO stability for a system with bounded desired velocity. Consequently, controllers like the damping-based approach in \parencite{kronander2015passive} do not require an energy tank. Yet, introducing an adaptive proportional term $\mathcal{K}$ can lead to instabilities \parencite{ferraguti2013tank, kronander2016stability}. 
Careful consideration of the controller design and stability analysis is necessary to ensure robust and safe performance in practical applications. 
Nevertheless, the passivity proof enables a broad range of time-varying damping controllers to be safely applied to robotic systems.

\iflong
\subsection{Point Mass Agents}
While the controller is limited to a point mass representation of the robot, its computational efficiency would allow multiple evaluations along the robot arm. This could be used to ensure full body control of the robot, hence improving the system's safety.
\fi

\subsection{Caution in Obstacle's Proximity}
In this work, we assume the obstacles' position to be precisely known. However, in many scenarios, the robot might have limited perception as it approaches an obstacle. Hence, the robot should be more compliant to enable safe workspace exploration rather than increasing the damping. Future work should explore how to combine these two opposing paradigms: safe control for avoidance yet cautious exploration.

\iflong
\appendix
\ifthesis
\section{Discrete Time Behavior} \label{sec:discrete_time_behavior}
\else
\subsection{Discrete Time Behavior} \label{sec:discrete_time_behavior}
\fi

So far, we've considered that the system is continuous time. 
This is a reasonable assumption if we have a high sampling time $\Delta t$, compared to the dynamics, i.e., $\| \Delta t \, \dot{\vecs \xi} \| \ll 1$.
However, any digital controller sends a discrete control signal. To reject the disturbances in the presence of obstacles, a high control force is exerted (while remaining within the control robot's limits), hence $\| \Delta t \, \vecs \tau^c \| \gg 1$. 
In this case, it is crucial to analyze the discrete-time system to guarantee stability, as high damping can lead to unstable behavior, see Figure~\ref{fig:discrete_controller_parameters_comparison}).

\begin{figure}[htbp]
\centering
\ifthesis
  \includegraphics[width=0.7\columnwidth]{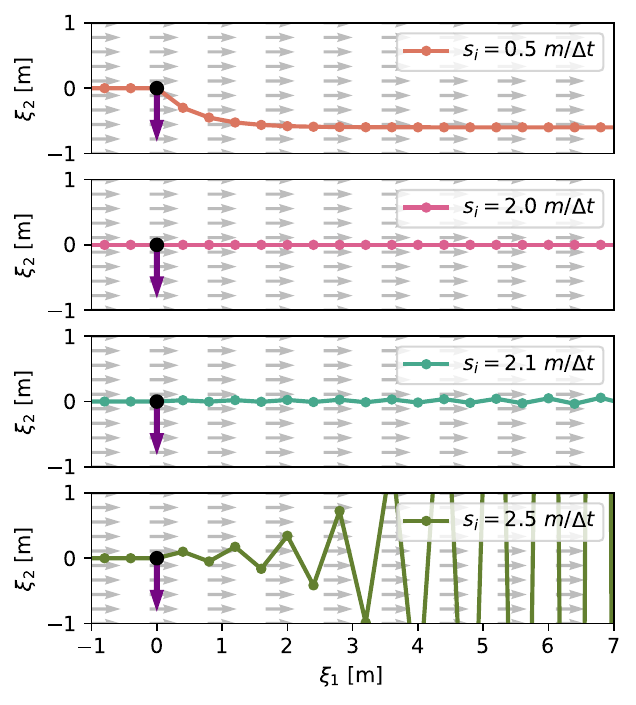}
  \else
  \includegraphics[width=\columnwidth]{figures/discrete_controller_parameters_comparison}
  \fi
  \caption{An agent has the desired velocity $\vect f(\vecs \xi) = [1, 0]^T$ (gray arrow), and a control matrix $\matd{D}$ with damping values equal in all directions and the smallest eigenvalue of the inertia matrix $m$. 
 The agent is disturbed (purple arrow) position $\vecs \xi_0 = [0, 0]^T$ and has a velocity of  $\vecs \xi_1 = [1, 1]^T$ after the impact. \\
  High damping leads to unstable behavior with increasingly high oscillations. Conversely, the lowest value leads to more deviation from the initial straight line resulting from the higher impedance. The critical value of $s_i = 2.0 m / \Delta t$ results in stable behavior with immediate correction to the desired velocity.}
  \label{fig:discrete_controller_parameters_comparison}
\end{figure}

For the discrete-time system, the position and velocity of the agent evolve as follows:
\begin{equation}
	\begin{bmatrix}
	 \vecs \xi_{t+1} \\ \dot{\vecs \xi}_{t+1}
	\end{bmatrix}
	=
	\begin{bmatrix}
	 \vecs \xi_{t} \\ \dot{\vecs \xi}_{t}
	\end{bmatrix}
	+ 
	\Delta t 
	\begin{bmatrix}
		\dot{\vecs \xi}_{t} \\ \ddot{\vecs \xi}_t 
	\end{bmatrix}
	\label{eq:discrete_time_dynamics}
\end{equation}

\begin{lemma}
	Let us consider a discrete-time system with the state as given in Eq.~\eqref{eq:discrete_time_dynamics}, and is governed by the controller from Eq.~\eqref{eq:control_command} and damping matrix $\matd{D}$ defined in Eq.~\eqref{eq:damping_summation}.
	The system is BIBO (bounded-input, bounded-stable) with respect input the desired velocity $\vect f(\vecs \xi)$ the velocity, and as output the agent's velocity $\dot{\vecs \xi}$ such that $\lim_{t \rightarrow \infty} \| \dot{\vect \xi} \| < \infty$, if all damping values are limited as $s_{d} \leq 2 \min \left( \text{eig}\left(\matd{M}\right)  \right) / \Delta t$ with $d=1, ..., N$.
\end{lemma}

\begin{proof}
The evolution of the discrete-time feedback system is given as:
\begin{equation}
	\begin{split}
	& \begin{bmatrix}
	 \vecs \xi_{t+1} \\ \dot{\vecs \xi}_{t+1}
	\end{bmatrix}
	=
	\begin{bmatrix}
		\vect \xi_t + \Delta t  \; \dot{\vect \xi}_t \\ \
		\dot{\vecs \xi}_t + \Delta t \, \matd{M}^{-1} \left( \matd{D} \left( \vect f(\vecs \xi_t) - \dot{\vecs \xi}_t \right) - \matd{C} \right)
	\end{bmatrix} \\
	&  = 
	\begin{bmatrix}
		\matr{I} & \matr{I} \Delta t \\
		\vect{0} & \matr{I} - \Delta t \matd{M}^{-1} \matd D 
	\end{bmatrix}
	\begin{bmatrix}
		\vect \xi_t \\ \dot{\vect \xi}_t
	\end{bmatrix}
	+ \begin{bmatrix}
		\vect{0} \\ 
		\Delta t \, \matd{M}^{-1} \matd{D} 
	\end{bmatrix}
	\hat{\vect f}(\vecs \xi_t) 
	\label{eq:discrete_time_proof}
	\end{split}
\end{equation}
where $\hat{\vect f}(\vecs \xi_t) = \vect f(\vecs \xi_t) - \matd{D}^{-1} C$.  The dependency on the state $\vecs \xi$ of the matrices $\matd{M}$, $\matd{D}$, and $\matd{C}$ are omitted for brevity. $\matr{I} \in \mathbb{R}^{N \times N}$ denotes the identity matrix.
As we look at global stability, we look at the updated velocity $\hat{\vect f}(\vecs \xi_t) = \vect f(\vecs \xi_t) - \matd{D}^{-1} \matd{C}$. Since the Coriolis force is bounded, it follows that if the system is BIBO stable for $\hat{\vect f}(\vect \xi_t)$, then it is also BIBO stable for $\vect f(\vecs \xi_t)$ 

BIBO stability of a discrete-time system requires that all the eigenvalues of the feedback matrix are smaller or equal to one \parencite{friedland2012control}.
The eigenvalues of the above feedback system are given as:
\begin{equation}
	\vect \lambda_1 = \text{eig}(\matr{I}) \qquad \vect \lambda_2 = \text{eig} \left(\matr{I} - \Delta t \, \matd{M}^{-1} \matd{D} \right)
\end{equation}
where both $\vect \lambda_1 \in \mathbb{R}^N$ and $\vect \lambda_2 \in \mathbb{R}^N$ denote a vector containing $N$ eigenvalues.
All eigenvalues of $\vect \lambda_{1, i} = 1, \; i = 1, ..., N$, and enable a stable system. 
The second set of eigenvalues $\vect \lambda_2$  depends on the passive control term. 
However, since $\matd{M}$ and $\matd{D}$ are both positive definite, the eigenvalues are positive:
\begin{equation}
	\matd{M} > 0 \; , \;\; \matd{D} > 0 
	\qquad \Rightarrow \qquad
	\vect \lambda_{2, i} > 0 \quad i = 1, ..., N
\end{equation}

Hence, we only have to consider the upper limit, and the system is stable if the maximum eigenvalue is limited to:
\begin{equation}
	\max \left(\vect \lambda_2 \right) \leq 1 
	\quad \Rightarrow \quad
	s_{i} \leq 2 \min \Bigl( \text{eig}\bigl(\matd{M}\bigr)  \Bigr) / \Delta t
\end{equation}
\end{proof}

The stability guarantees provide BIBO stability, hence the boundedness of the output. 
Since the first eigenvalues equal one, there is no global asymptotic convergence. 
In fact, in the system from \eqref{eq:discrete_time_proof}, it can be observed that when the input dynamics are zero, i.e., $\vect f(\vecs \xi) = \vect 0$, the system immediately stops. But there is no convergence to a specific position.
However, in most cases, the desired system should only reach zero at the attractor position, and hence, we expect convergence to such a point.

\ifthesis
Yet, asymptotic stability is not guaranteed for general nonlinear input dynamics. Especially for dynamics with high curvature and low damping value, the final trajectory can end up in limit cycles ((Figure~\ref{fig:discrete_controller_parameters_comparison_stable}).
\fi

In practice, it is useful to use a large value for $s^{o}$ as it rejects disturbances towards the obstacles, and lower values for the dynamics following $s^{f}$ and damping in general direction $s^{c}$. Hence, we propose $s^{o} = 2.0 m / \Delta t$, $s^{f} = 1.0 / \Delta t$, and $s^c = 0.1 / \Delta t$.

\begin{figure}[htbp]
\centering
\ifthesis
  \includegraphics[width=.7\columnwidth]{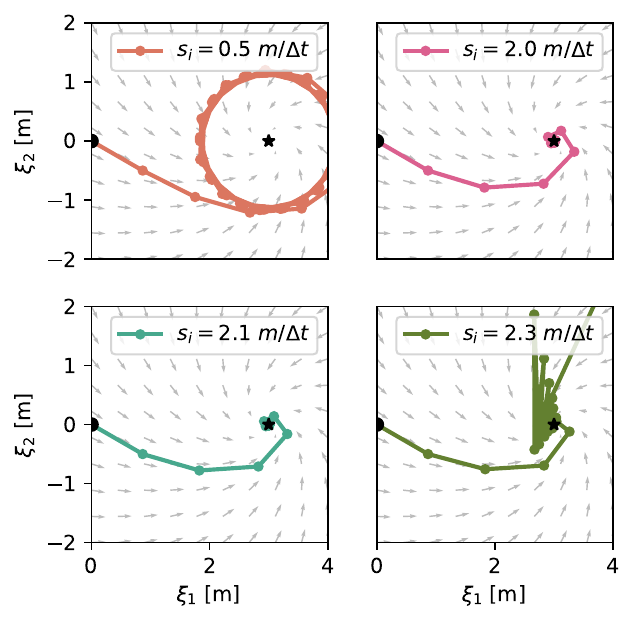}
  \else
  \includegraphics[width=\columnwidth]{figures/discrete_controller_parameters_comparison_stable}
  \fi
\caption{An agent with the desired dynamics of
$\vect f(\vecs \xi) = \matr R(\pi / 6) (\vect \xi  - \vect \xi^a)$ where $\matr R(\cdot) \in \mathbb{R}^{N \times N}$ is the rotation matrix, and $m$ is the mass of the agent. We assume equal damping values $s_i$.
The controller with a critical stiffness of $2.0 m / \Delta t$ leads to fast convergence and a stable system. With lower damping (top left), there is a large drift of the system, which converges to a limit cycle. 
The high damping of $2.3 m / \Delta t$ leads to an unstable system. 
Interestingly, with damping of $2.1 m / \Delta t$ in combination with the nonlinear dynamics, we obtain a visibly stable system.}
  \label{fig:discrete_controller_parameters_comparison_stable}
\end{figure}

\fi

\renewcommand*{\bibfont}{\footnotesize}
\printbibliography

@article{slotine1987adaptive,
  title={On the adaptive control of robot manipulators},
  author={Slotine, Jean-Jacques E and Li, Weiping},
  journal={The international journal of robotics research},
  volume={6},
  number={3},
  pages={49--59},
  year={1987},
  publisher={Sage Publications Sage CA: Thousand Oaks, CA}
}

@misc{passive2024huber_extended,
  title={Passive Obstacle Aware Control
to Follow Desired Velocities},
  author={Huber, Lukas and Trinca, Thibaud and Slotine, Jean-Jacques and Billard, Aude},
  journal={arXiv preprint arXiv},
  howpublished = "\url{http://arxiv.org/abs/2405.05669}",
  year={2024}
}

@inproceedings{bony1969principe,
  title={Principe du maximum, in{\'e}galit{\'e} de Harnack et unicit{\'e} du probleme de Cauchy pour les op{\'e}rateurs elliptiques d{\'e}g{\'e}n{\'e}r{\'e}s},
  author={Bony, Jean-Michel},
  booktitle={Annales de l'institut Fourier},
  volume={19},
  number={1},
  pages={277--304},
  year={1969}
}

@article{connolly1997harmonic,
  title={Harmonic functions and collision probabilities},
  author={Connolly, Christopher I},
  journal={The International Journal of Robotics Research},
  volume={16},
  number={4},
  pages={497--507},
  year={1997},
  publisher={Sage Publications Sage CA: Thousand Oaks, CA}
}

@inproceedings{glosser1994implementation,
  title={The implementation of a natural admittance controller on an industrial manipulator},
  author={Glosser, Gregory D and Newman, Wyatt S},
  booktitle={Proceedings of the 1994 IEEE International Conference on Robotics and Automation},
  pages={1209--1215},
  year={1994},
  organization={IEEE}
}

@inproceedings{haddadin2010real,
  title={Real-time reactive motion generation based on variable attractor dynamics and shaped velocities},
  author={Haddadin, Sami and Urbanek, Holger and Parusel, Sven and Burschka, Darius and Ro{\ss}mann, J{\"u}rgen and Albu-Sch{\"a}ffer, Alin and Hirzinger, Gerd},
  booktitle={2010 IEEE/RSJ International Conference on Intelligent Robots and Systems},
  pages={3109--3116},
  year={2010},
  organization={IEEE}
}

@book{friedland2012control,
  title={Control system design: an introduction to state-space methods},
  author={Friedland, Bernard},
  year={2012},
  publisher={Courier Corporation}
}

@article{loizou2022mobile,
  title={Mobile Robot Navigation Functions Tuned by Sensor Readings in Partially Known Environments},
  author={Loizou, Savvas and Rimon, Elon},
  journal={IEEE Robotics and Automation Letters},
  year={2022},
  publisher={IEEE}
}

@article{koditschek1990robot,
 title={Robot navigation functions on manifolds with boundary},
  author={Koditschek, Daniel E and Rimon, Elon},
  journal={Advances in applied mathematics},
  volume={11},
  number={4},
  pages={412--442},
  year={1990},
  publisher={Elsevier}
}

@article{huber2023avoidance,
  title={Avoidance of Concave Obstacles through Rotation of Nonlinear Dynamics},
  author={Huber, Lukas and Slotine, Jean-Jacques and Billard, Aude},
  journal={IEEE Transactions on Robotics},
  year={2023},
  publisher={IEEE}
}

@article{xie2020geometric,
  title={Geometric fabrics for the acceleration-based design of robotic motion},
  author={Xie, Mandy and Van Wyk, Karl and Li, Anqi and Rana, Muhammad Asif and Wan, Qian and Fox, Dieter and Boots, Byron and Ratliff, Nathan},
  journal={arXiv preprint arXiv:2010.14750},
  year={2020}
}

@inproceedings{singh1996real,
  title={Real-time robot motion control with circulatory fields},
  author={Singh, Leena and Stephanou, Harry and Wen, John},
  booktitle={Proceedings of IEEE International Conference on Robotics and Automation},
  volume={3},
  pages={2737--2742},
  year={1996},
  organization={IEEE}
}

@article{udwadia2003new,
  title={A new perspective on the tracking control of nonlinear structural and mechanical systems},
  author={Udwadia, FE},
  journal={Proceedings of the Royal Society of London. Series A: Mathematical, Physical and Engineering Sciences},
  volume={459},
  number={2035},
  pages={1783--1800},
  year={2003},
  publisher={The Royal Society}
}

@article{fujiki2022series,
  title={Series admittance--impedance controller for more robust and stable extension of force control},
  author={Fujiki, Takuto and Tahara, Kenji},
  journal={ROBOMECH Journal},
  volume={9},
  number={1},
  pages={23},
  year={2022},
  publisher={Springer}
}

@article{kronander2016stability,
  title={Stability considerations for variable impedance control},
  author={Kronander, Klas and Billard, Aude},
  journal={IEEE Transactions on Robotics},
  volume={32},
  number={5},
  pages={1298--1305},
  year={2016},
  publisher={IEEE}
}

@article{brock2002elastic,
  title={Elastic strips: A framework for motion generation in human environments},
  author={Brock, Oliver and Khatib, Oussama},
  journal={The International Journal of Robotics Research},
  volume={21},
  number={12},
  pages={1031--1052},
  year={2002},
  publisher={SAGE Publications Sage UK: London, England}
}

@inproceedings{ferraguti2013tank,
  title={A tank-based approach to impedance control with variable stiffness},
  author={Ferraguti, Federica and Secchi, Cristian and Fantuzzi, Cesare},
  booktitle={2013 IEEE international conference on robotics and automation},
  pages={4948--4953},
  year={2013},
  organization={IEEE}
}

@article{vanderborght2013variable,
  title={Variable impedance actuators: A review},
  author={Vanderborght, Bram and Albu-Sch{\"a}ffer, Alin and Bicchi, Antonio and Burdet, Etienne and Caldwell, Darwin G and Carloni, Raffaella and Catalano, Manuel and Eiberger, Oliver and Friedl, Werner and Ganesh, Ganesh and others},
  journal={Robotics and autonomous systems},
  volume={61},
  number={12},
  pages={1601--1614},
  year={2013},
  publisher={Elsevier}
}

@inproceedings{bylard2021composable,
  title={Composable geometric motion policies using multi-task pullback bundle dynamical systems},
  author={Bylard, Andrew and Bonalli, Riccardo and Pavone, Marco},
  booktitle={2021 IEEE International Conference on Robotics and Automation (ICRA)},
  pages={7464--7470},
  year={2021},
  organization={IEEE}
}

@inproceedings{ratliff2021generalized,
  title={Generalized nonlinear and finsler geometry for robotics},
  author={Ratliff, Nathan D and Van Wyk, Karl and Xie, Mandy and Li, Anqi and Rana, Muhammad Asif},
  booktitle={2021 IEEE International Conference on Robotics and Automation (ICRA)},
  pages={10206--10212},
  year={2021},
  organization={IEEE}
}

@article{peters2008unifying,
  title={A unifying framework for robot control with redundant DOFs},
  author={Peters, Jan and Mistry, Michael and Udwadia, Firdaus and Nakanishi, Jun and Schaal, Stefan},
  journal={Autonomous Robots},
  volume={24},
  pages={1--12},
  year={2008},
  publisher={Springer}
}

@article{takegaki1981new,
  title={A new feedback method for dynamic control of manipulators},
  author={Takegaki, Morikazu and Arimoto, Suguru},
  year={1981}
}

@inproceedings{hogan1984impedance,
  title={Impedance control: An approach to manipulation},
  author={Hogan, Neville},
  booktitle={1984 American control conference},
  pages={304--313},
  year={1984},
  organization={IEEE}
}

@article{hogan1985impedance,
  title={Impedance control: An approach to manipulation: Part II—Implementation},
  author={Hogan, Neville},
  year={1985}
}

@article{abu2020variable,
  title={Variable impedance control and learning—a review},
  author={Abu-Dakka, Fares J and Saveriano, Matteo},
  journal={Frontiers in Robotics and AI},
  volume={7},
  pages={590681},
  year={2020},
  publisher={Frontiers Media SA}
}

@article{stramigioli2005sampled,
  title={Sampled data systems passivity and discrete port-hamiltonian systems},
  author={Stramigioli, Stefano and Secchi, Cristian and van der Schaft, Arjan J and Fantuzzi, Cesare},
  journal={IEEE Transactions on Robotics},
  volume={21},
  number={4},
  pages={574--587},
  year={2005},
  publisher={IEEE}
}

@article{lee2010passive,
  title={Passive-set-position-modulation framework for interactive robotic systems},
  author={Lee, Dongjun and Huang, Ke},
  journal={IEEE Transactions on Robotics},
  volume={26},
  number={2},
  pages={354--369},
  year={2010},
  publisher={IEEE}
}

@inproceedings{cheng2020rmp,
  title={RMP flow: A Computational Graph for Automatic Motion Policy Generation},
  author={Cheng, Ching-An and Mukadam, Mustafa and Issac, Jan and Birchfield, Stan and Fox, Dieter and Boots, Byron and Ratliff, Nathan},
  booktitle={Algorithmic Foundations of Robotics XIII: Proceedings of the 13th Workshop on the Algorithmic Foundations of Robotics 13},
  pages={441--457},
  year={2020},
  organization={Springer}
}

@article{van2022geometric,
  title={Geometric fabrics: Generalizing classical mechanics to capture the physics of behavior},
  author={Van Wyk, Karl and Xie, Mandy and Li, Anqi and Rana, Muhammad Asif and Babich, Buck and Peele, Bryan and Wan, Qian and Akinola, Iretiayo and Sundaralingam, Balakumar and Fox, Dieter and others},
  journal={IEEE Robotics and Automation Letters},
  volume={7},
  number={2},
  pages={3202--3209},
  year={2022},
  publisher={IEEE}
}

@article{khatib1987unified,
  title={A unified approach for motion and force control of robot manipulators: The operational space formulation},
  author={Khatib, Oussama},
  journal={IEEE Journal on Robotics and Automation},
  volume={3},
  number={1},
  pages={43--53},
  year={1987},
  publisher={IEEE}
}

@article{albu2007unified,
  title={A unified passivity-based control framework for position, torque and impedance control of flexible joint robots},
  author={Albu-Sch{\"a}ffer, Alin and Ott, Christian and Hirzinger, Gerd},
  journal={The international journal of robotics research},
  volume={26},
  number={1},
  pages={23--39},
  year={2007},
  publisher={Sage Publications Sage CA: Thousand Oaks, CA}
}

@article{haddadin2011dynamic,
  title={Dynamic motion planning for robots in partially unknown environments},
  author={Haddadin, Sami and Belder, Rico and Albu-Sch{\"a}ffer, Alin},
  journal={IFAC Proceedings Volumes},
  volume={44},
  number={1},
  pages={6842--6850},
  year={2011},
  publisher={Elsevier}
}

@article{huber2019avoidance,
  title={Avoidance of convex and concave obstacles with convergence ensured through contraction},
  author={Huber, Lukas and Billard, Aude and Slotine, Jean-Jacques},
  journal={IEEE Robotics and Automation Letters},
  volume={4},
  number={2},
  pages={1462--1469},
  year={2019},
  publisher={IEEE}
}

@article{huber2022avoiding,
  title={Avoiding dense and dynamic obstacles in enclosed spaces: Application to moving in crowds},
  author={Huber, Lukas and Slotine, Jean-Jacques and Billard, Aude},
  journal={IEEE Transactions on Robotics},
  volume={38},
  number={5},
  pages={3113--3132},
  year={2022},
  publisher={IEEE}
}

@article{kronander2015passive,
  title={Passive interaction control with dynamical systems},
  author={Kronander, Klas and Billard, Aude},
  journal={IEEE Robotics and Automation Letters},
  volume={1},
  number={1},
  pages={106--113},
  year={2015},
  publisher={IEEE}
}

@inproceedings{tulbure2020closing,
  title={Closing the loop: Real-time perception and control for robust collision avoidance with occluded obstacles},
  author={Tulbure, Andreea and Khatib, Oussama},
  booktitle={2020 IEEE/RSJ International Conference on Intelligent Robots and Systems (IROS)},
  pages={5700--5707},
  year={2020},
  organization={IEEE}
}

@book{sepulchre2012constructive,
  title={Constructive nonlinear control},
  author={Sepulchre, Rodolphe and Jankovic, Mrdjan and Kokotovic, Petar V},
  year={2012},
  publisher={Springer Science \& Business Media}
}

@article{willems1972dissipative,
  title={Dissipative dynamical systems part I: General theory},
  author={Willems, Jan C},
  journal={Archive for rational mechanics and analysis},
  volume={45},
  number={5},
  pages={321--351},
  year={1972},
  publisher={Springer}
}

@inproceedings{gribovskaya2011motion,
  title={Motion learning and adaptive impedance for robot control during physical interaction with humans},
  author={Gribovskaya, Elena and Kheddar, Abderrahmane and Billard, Aude},
  booktitle={2011 IEEE International Conference on Robotics and Automation},
  pages={4326--4332},
  year={2011},
  organization={IEEE}
}

@article{li1999passive,
  title={Passive velocity field control of mechanical manipulators},
  author={Li, Perry Y and Horowitz, Roberto},
  journal={IEEE Transactions on robotics and automation},
  volume={15},
  number={4},
  pages={751--763},
  year={1999},
  publisher={IEEE}
}

@article{duindam2004passive,
  title={Passive compensation of nonlinear robot dynamics},
  author={Duindam, Vincent and Stramigioli, Stefano and Scherpen, Jacquelien MA},
  journal={IEEE transactions on robotics and automation},
  volume={20},
  number={3},
  pages={480--488},
  year={2004},
  publisher={IEEE}
}

@inproceedings{kishi2003passive,
  title={Passive impedance control with time-varying impedance center},
  author={Kishi, Yasuo and Luo, Zhi Wei and Asano, Fumihiko and Hosoe, Shigeyuki},
  booktitle={Proceedings 2003 IEEE International Symposium on Computational Intelligence in Robotics and Automation. Computational Intelligence in Robotics and Automation for the New Millennium (Cat. No. 03EX694)},
  volume={3},
  pages={1207--1212},
  year={2003},
  organization={IEEE}
}

\end{document}